\documentclass[a4paper]{article}
\usepackage[a4paper,margin=2cm]{geometry}

\usepackage[round]{natbib}

\usepackage[utf8]{inputenc} 
\usepackage[T1]{fontenc}    
\usepackage{hyperref}       
\usepackage{url}            
\usepackage{booktabs}       
\usepackage{amsfonts}       
\usepackage{nicefrac}       
\usepackage{microtype}      
\usepackage{enumitem}
\setenumerate[1]{label=(\arabic*)}

\usepackage{subfigure}
\usepackage{wrapfig}

\usepackage{graphicx}
\hypersetup{colorlinks=true, linkcolor=magenta, citecolor=blue}
\usepackage[ruled]{algorithm2e} 
\usepackage{amsmath,amsthm,amssymb}
\usepackage{bm}
\usepackage{color}
\newcommand{\refalg}[1]{Algorithm~\ref{alg:#1}}
\newcommand{\refeq}[1]{\eqref{eq:#1}}

\newcommand{\refsupp}[1]{Supplementary Appendix~\ref{sec:#1}}
\newcommand{\refsec}[1]{\S~\ref{sec:#1}}
\newcommand{\refprop}[1]{Proposition~\ref{prop:#1}}

\newcommand{\refthm}[1]{Theorem~\ref{thm:#1}}
\newcommand{\reflem}[1]{Lemma~\ref{lem:#1}}

\newtheorem{thm}{Theorem}[section]
\newtheorem{lem}[thm]{Lemma}
\newtheorem{prop}[thm]{Proposition}
\newtheorem{dfn}[thm]{Definition}

\newcommand{\EE}{\mathbb{E}}
\newcommand{\dd}{\mathrm{d}}
\newcommand{\argmax}{\mathop{\rm arg~max}\limits}
\newcommand{\argmin}{\mathop{\rm arg~min}\limits}

\newcommand{\wdot}{\,\cdot\,}
\newcommand{\hull}{\mathsf{span}}
\newcommand{\conv}{\mathsf{conv}}
\newcommand{\iprod}[1]{\langle#1\rangle}

\newcommand{\diam}{\mathsf{diam}}

\newcommand{\mmd}{\mathsf{MMD}}

\newcommand{\PP}{\mathbb{P}}
\newcommand{\QQ}{\mathbb{Q}}

\newcommand{\CC}{\mathbb{C}}
\newcommand{\RR}{\mathbb{R}}
\newcommand{\NN}{\mathbb{N}}

\newcommand{\F}{\mathcal{F}}

\newcommand{\G}{\mathcal{G}}
\renewcommand{\H}{\mathcal{H}}
\newcommand{\M}{\mathcal{M}}

\newcommand{\B}{\mathcal{B}} 
\newcommand{\V}{\mathcal{V}}
\newcommand{\X}{\mathcal{X}}
\newcommand{\Z}{\mathcal{Z}}

\newcommand{\rad}{\mathfrak{R}}
\newcommand{\radhat}{\widehat{\mathfrak{R}}}
\newcommand{\gss}{\mathfrak{G}}
\newcommand{\gsshat}{\widehat{\mathfrak{G}}}

\newcommand{\dom}{\mathcal{D}}
\newcommand{\err}{\mathcal{E}}

\newcommand{\muhat}{\widehat{\mu}}

\newcommand{\Lhat}{\widehat{L}}

\newcommand{\fhat}{\widehat{f}}

\newcommand{\dagphi}{\varphi^\dag}

\title{Fast Approximation and Estimation Bounds\\of Kernel Quadrature for Infinitely Wide Models}

\author{%
  Sho Sonoda \\
  RIKEN AIP\\
  Tokyo 103--0027, Japan\\
  \texttt{sho.sonoda@riken.jp}
}

\date{June 11, 2020}

\begin{document}

\maketitle

\begin{abstract}
An infinitely wide model is a weighted integration $\int \varphi(x,v) \dd \mu(v)$ of feature maps. This model excels at handling an infinite number of features, and thus it has been adopted to the theoretical study of deep learning. Kernel quadrature is a kernel-based numerical integration scheme developed for fast approximation of expectations $\int f(x) \dd p(x)$. In this study, with regarding the weight $\mu$ as a signed (or complex/vector-valued) distribution of parameters, we develop the general kernel quadrature (GKQ) for parameter distributions. The proposed method can achieve a fast approximation rate $O(e^{-p})$ with parameter number $p$, which is faster than the traditional Barron's rate, and a fast estimation rate $\widetilde{O}(1/n)$ with sample size $n$. As a result, we have obtained a new norm-based complexity measure for infinitely wide models. Since the GKQ implicitly conducts the empirical risk minimization, we can understand that the complexity measure also reflects the generalization performance in the gradient learning setup.
\end{abstract}

{
\newcommand{\yy}{\mathbf{y}}
\newcommand{\xx}{\mathbf{x}}
\newcommand{\ww}{\mathbf{w}}
\newcommand{\vv}{\mathbf{v}}
\renewcommand{\SS}{\mathbf{S}}
\newcommand{\xxi}{\bm{\xi}}
\newcommand{\ddelta}{\bm{\delta}}
\section{Introduction} \label{sec:intro}

In this study, we consider the finite approximation problem of an infinitely wide model. This model covers a variety of modern machine learning methods, for example, such as random features \citep{Rahimi2008,Rahimi2009}, ensemble learning \citep{Schapire2012}, Bayesian neural networks \citep{Neal1996}, and kernel machines \citep{Scholkopf2001}. Recently, it has also been employed in the theoretical study of deep learning \citep{Nitanda2017,Rotskoff2018,Mei2018,Sirignano2020,Chizat2018,Jacot2018,Frankle2019,Belkin2019a,Hastie2019}. Despite a large number of parameters, called the \emph{over-parametrized regime}, deep learning achieves high generalization performance. Providing a better complexity measure for estimating the generalization performance of over-parametrized models are a recent trend in deep learning study.

The integral representation $S[\mu](x) = \int \varphi(x;v)\dd \mu(v)$ is a handy expression for infinitely wide models. Compared to a finite model $g(x;\theta_p) = \sum_{j=1}^p w_j \varphi(x;v_j)$, a measure $\mu$ ``indexed by $v$'' is an infinite alternative of a finite set of parameters $(w_j,v_j)$ ``indexed by $j$''. We call the measure $\mu$ a \emph{parameter distribution}. We remark (1) that we do not need to restrict $\mu$ to be a probability density, and we can assume $\mu$ to be either signed/complex/vector-valued measures; (2) that we can reproduce a finite model $g(x;\theta_p)$ by using Dirac measures $\delta_v$ and letting $\mu_p = \sum_{j=1}^p w_j \delta_{v_j}$;
and thus, (3) that we do not need to pre-fix the parameter number $p$ before/during training.
Thus, the integral representation is \emph{not} a counterpart of the finite models, but it is an extension of the finite models.

Through the finite approximation problem, we aim to investigate a new complexity measure for infinite models. One way to measure the complexity of an infinite model is to construct a finite approximation $g(x;\theta_p) = \sum_{j=1}^p w_j \varphi(x;v_j)$ and count the number $p$ of nodes required for the approximation. The idea of finite approximation is not new but has a long history. For example, the \emph{smoothness} of a function reflects the complexity in the approximation by algebraic polynomials, and the \emph{decay property} in the frequency domain reflects the complexity in the approximation by trigonometric polynomials. These complexities are derived based on a fixed basis.

One specific goal of this study is the so-called \emph{Barron's rate} \citep{Barron1993,Kurkova2012}, a traditional approximation rate $O(1/\sqrt{p})$. Although this is the same rate as Monte Carlo methods, it has a strong advantage that it does not depend on the input dimension. In contrast, the so-called \emph{Jackson's rate} \citep{DeVore1993}, another more historical rate $O(p^{-s/m})$, depends on the dimension $m$ and smoothness $s$, which results in the \emph{curse of dimensionality}. The difference stems from that Barron's theory assumes an adaptive basis, whereas Jackson's theory assumes a fixed basis.
Since neural networks have an adaptive basis, we consider in the framework of Barron's theory.
In this study, we apply the concept of kernel quadrature to construct the finite approximation, which is also based on the strategy of adaptive node selection.

Kernel quadrature is a kernel-based numerical integration method developed for computing expectations $\int f(x) \dd P(x)$ \citep{Briol2015a}. It is closely related to quasi-Monte Carlo \citep{Rubinstein2016}, and thus known to converge faster than the ordinary Monte Carlo.
Between integral transforms and expectations, there are at least three gaps: First, $\mu$ is no more a probability measure but a vector-valued measure. Second, the approximant is no more a constant but a function. Finally, $\mu$ parametrizes a function $S[\mu]$. To address these gaps, we reconstruct the \emph{KME for vector-valued measures}, and establish the \emph{KQ in a function space}. Through the reconstruction, we expect to find an appropriate complexity. Fortunately, under mild conditions, we have shown that our proposed kernel quadrature method implicitly minimizes the empirical risk. Therefore, the proposed complexity also estimates the generalization error of infinite models learned by gradient descent.

One of the technical difficulties in handling parameter distributions is that \emph{the correspondence between a parameter distribution $\mu$ and the function $f = S[\mu]$ is not always 1-to-1.}
In other words, the map $S : \mu \mapsto f$ is not always injective. In particular, it is usually true for neural networks. (See \refsupp{eg.injective} for more detail.)
\citet{Bach2017b} and \citet{Chizat2018} employed the total variation norm and Wasserstein distance for parameter distributions.
However, they are too strong to handle parameter distributions for neural networks,
because if two different distributions $\mu$ and $\nu$ indicates the same function: $S[\mu] = S[\nu]$,
then both distances unnecessarily distinguish distributions: $d( \mu, \nu ) > 0$, which further results in another non-convexity of the training.
(Here, we note that they introduced sparse regularizations to convexify the training problem.)
Obviously, it is preferable to identify $\mu$ and $\nu$ if $S[\mu] = S[\nu]$. In other words, rather than the raw space $\M$ of measures, we should work on the quotient space $\M / \ker S$.
To address this issue, we come to an idea of the \emph{unitary kernel} that can vanish the null space of $S$ and induce a natural geometry to the parameter space.

\paragraph{Contributions.}
The main contributions of this study are in (1) investigating an appropriate metric (KME and MMD) for parameter distributions, (2) proposing a few approximation algorithms (GKQ, EKQ and UKQ), and (3) providing fast approximation and estimation error bounds ($O(e^{-p})$ and $O(1/n)$ respectively) and corresponding complexity measure $\| K[\mu] \|_K$.

\paragraph{Related Works.}
Thus far, many authors attempted to estimate the parameter distribution $\mu$.
We note that from the view point of linear algebra, or functional analysis, 
training a single layer model $S[\mu]$ is equivalent to estimating a pseudo-inverse operator $S^\dag[f]$.
\emph{Convex neural networks} \citep{LeRoux2006,LeRoux2006b,Bach2017a} aimed to conduct conditional gradient in practice, and \citet{Bach2017b} pointed out that it is revealed to be ``intractable in practice''. The intractability stems from a non-convex subproblem in the optimization algorithm.
\emph{Wasserstein gradient flow} \citep{Nitanda2017,Chizat2018} also aimed to  estimate parameter distributions by introducing the Wasserstein metric on the space of parameter distributions. The ideas are natural, but as we describe later, the convergence may be too strong for parameter distributions.
In addition, parameter numbers $p$ should be fixed beforehand.
\citet{Sonoda2014} proposed a \emph{numerical integration} method for approximating parameter distributions. They computed the \emph{ridgelet transform} $R$ \citep{Murata1996,Candes.HA,Sonoda2017}, or an explicit expression of the pseudo-inverse operator $S^\dag$ for neural networks, and proposed a simple Monte-Carlo sampling method for training shallow neural networks. By the uniform law of large numbers, their method can  converge at $O(1/\sqrt{p})$ in $L^\infty$ (see \refthm{ulln.sir}). 

\paragraph{Notation.}
$\overline{z}$ denotes the complex conjugate of a complex number $z$.
$\log_+(x)$ denotes the truncated logarithm $\max\{ 1, \log x\}$.
$|\mu|$ denotes the total variation measure of a real/complex/vector measure $\mu$.
$\mu[f]$ denotes the integration $\int f \dd \mu$ of a function $f$ by a measure $\mu$.
$\PP[f]$ denotes the expectation $\EE_{X \sim \PP}[f(X)]$ of a function $f$ by a probability distribution $\PP$.
$N(\mu,\sigma^2)$ denotes the normal distribution with mean $\mu$ and variance $\sigma^2$.
}

\section{Setup}

\paragraph{Learning model.}
Let $\X \subset \RR^m$ be the input data space.
We fix a data distribution $\PP$ on $\X$, and assume that it has the full support over $\X$.
By $L^2(\PP)$, we denote the Hilbert space of complex-valued $L^2$-functions on $\X$ with base measure $\dd \PP$.

Let $\V \subset \RR^d$ be the parameter space.
By $\M$ we denote the space of all complex Borel measures $\mu$ on $\V$.
We call an element $\mu \in \M$ a \emph{parameter distribution}.

Let $\varphi : \X \times \V \to \CC$ be the Borel measurable feature map.
To the end of this study, we consider a shallow model with expression
\begin{align}
    S[\mu] := \int_{\V} \varphi( \wdot ;v) \dd \mu(v), \quad \mu \in \M.
\end{align}
Let $\F$ be a Banach (or Hilbert) space on $\X$ with norm $\| \cdot \|_\F$ (or inner product $\iprod{\cdot, \cdot}_\F$).
We consider a subclass $\M_\F$ where $\| S[\mu] \|_\F < \infty$ for any $\mu \in \M_\F$.

We remark that both the complex and Borel assumptions on parameter distributions $\M$ are natural for our study
because $\mu \in \M$ will be identified as, for example, a Fourier spectrum, which is typically complex-valued;
and a finite sum of Dirac measures, which is not a function but a Borel measure.
In the supplementary materials, we consider a more general case where $\mu$ is vector-valued, but in the following, for the sake of simplicity, we only consider the special case where $\mu$ is complex-valued. 

\paragraph{Learning problem.}
Let $D_n := \{ (x_i, y_i) \}_{i=1}^n \subset \X \times \CC$ be a dataset of $n$ examples.
We assume that $x_i \sim \PP$ are i.i.d., and $y_i$ are obtained by the regression model $y_i = S[\mu^o](x_i) + \xi_i$ with an unknown parameter distribution $\mu^o \in \M$ and i.i.d. observation noise $\xi_i \sim N(0,\sigma^2)$.

Given a dataset $D_n$, the ultimate goal of supervised learning is to give a decision rule $A : D_n \mapsto f \in \F$ that minimizes the generalization error, which is often given by an expected risk $L[f] := \EE_{X,Y}[ \ell( Y, f(X) ) ]$ with a certain loss function $\ell : \CC \times \CC \to [0,\infty)$.
In this study, we investigate a kernel quadrature approach to estimate $\mu$. 

\paragraph{Auxiliary metrics.}
One of the aims in this study is to investigate appropriate metrics for parameter distributions. We consider the \emph{total variation norm} $\| \mu \|_{TV} := \int_\V \dd |\mu|(v)$ and the \emph{induced norm} $\| \mu \|_\F := \| S[\mu] \|_\F$ for reference purposes. We write $\M_{TV} := \{ \mu \in \M \mid \| \mu \|_{TV} < \infty \}$ and $\M_\F := \{ \mu \in \M \mid \| S[\mu] \|_\F < \infty \}$ respectively.

The metric of parameter distributions determines the regularity of model $S$.
The induced norm is a natural (weaker) norm because $S : \M_\F \to \F$ is always continuous: $\| S[\mu] \|_\F = \| \mu \|_\F$.
In contrast, we need more assumptions for $S : \M_{TV} \to \F$ to be continuous.
For example, assume that the feature map $\varphi$ satisfies the regularity condition $C_{TV \to \F} := \sup_{v \in \V} \| \varphi(\wdot; v) \|_\F < \infty$. Then, $S : \M_{TV} \to \F$ is a bounded (i.e., Lipschitz continuous) linear operator because $\| S[\mu] \|_\F \le C_{TV \to \F} \| \mu \|_{TV}$.

The TV norm is related to the sparsity.
If $\mu$ has a density function $\varrho = \dd \mu / \dd v$, or the Radon–Nikodým derivative, then the total variation norm is reduced to the $L^1$-norm: $\| \mu \|_{TV} = \| \varrho \|_{L^1}$; and if $\mu$ is a sum of singular measures such as $\sum_{j=1}^p w_j \delta_{v_j}$, then the total variation norm is reduced to the $\ell^1$-norm: $\| \mu \|_{TV} = \sum_{j=1}^p |w_j|$.
In many cases, the TV norm is stronger than other topology. For example, the law of large numbers and the central limit theorem cannot hold in TV.

\section{General Kernel Quadrature}
We introduce the kernel mean embedding (KME) and general kernel quadrature (GKQ) for parameter distributions $\mu \in \M$, to construct a finite model $S[\mu_p](x) = \sum_{j=1}^p w_j \varphi(x;v_j)$ that approximates an infinite model $S[\mu]$ at a faster rate than the traditional Barron's rate $O(1/\sqrt{p})$.

\subsection{Quick Review of Kernel Quadrature for Probability Distributions}
Kernel quadrature (KQ) is a fast Monte-Carlo (MC) integration scheme that is developed for approximating an integral $I = \int_{\X} f(x) \dd \PP(x)$ by a finite sum $I_t = \sum_{s=1}^t w_s f(x_s)$ with weight $w_s \in \RR$ and samples $x_s \sim \PP$. Whereas the ordinary MC integration schemes converge at $O_p(1/\sqrt{t})$ by drawing samples $x_s$ in an independent manner, the KQ methods typically converge at a faster rate such as $O_p(1/t)$ by drawing samples $x_s$ in an incremental (i.e., dependent) manner.

In brief, the final output of KQ is an empirical measure $\PP_t = \sum_{s=1}^t w_s \delta_{x_s}$ that approximates $\PP$ in the maximum mean discrepancy (MMD) metric. The MMD is a kernel-based metric for probability distributions, which we will explain later. The approximation problem is formulated as a convex optimization problem that minimizes the MMD between $\PP_t$ and $\PP$ subject to $\PP_t$, and solved by the conditional gradient (CG) method. Since CG is an incremental method, KQ generates $t$-term approximator $\PP_t$ after $t$ iterations.

\paragraph{KME and MMD.}
Let $k : \X \times \X \to \RR$ be a measurable positive definite kernel function, and let $H$ be its associated reproducing kernel Hilbert space (RKHS).
To begin with, we define the kernel mean embedding (KME) of $\PP$ by $k[\PP](x) := \int_{\X} k(x',x) \dd \PP(x')$.
If $\PP$ satisfies the \emph{kernel moment condition}, $\int_{\X} \| k( \cdot, x) \|_k \dd \PP(x) < \infty$, then the image $k[\PP]$ belongs to $H$ and thus KME is an embedding. To be precise, $k$ is a mapping (an embedding) and the image $k[\PP]$ is called a \emph{mean element}. However, we call $k[\PP]$ a KME as long as there is no risk of confusion. Then, we define the MMD between probability distributions $\PP$ and $\QQ$ by an integral probability metric (IPM) \citep{Muller1997} $\mmd[\PP,\QQ] := \sup_{\| f \|_k \le 1} | \PP[f] - \QQ[f] |$. It is an important characterization that an MMD is equivalent to the distance between KMEs, namely $\mmd[\PP,\QQ] = \| k[\PP] - k[\QQ] \|_k$.
We refer to \citet{Muandet2017} for more details on the KME and MMD for probability measures.

\begin{wrapfigure}[15]{R}{0.55\textwidth}
    \vspace{-0.45cm}
    \begin{algorithm}[H]
    \SetKwInOut{Input}{input}\SetKwInOut{Output}{output}
    \Input{initial point $x_0 \in \dom$, quality $\alpha \ge 0$}
    Let $u_0 := x_0$\\
    \For{$s=1$ \KwTo $t$}{
    Let $\eta_s := 2 / (2+s)$\\
    Find $u_s \in \dom$ s.t.
    $u_s \cdot \nabla g( x_s )$\\
    \quad $\le \min_{u \in \dom} u \cdot \nabla g(x_s) + \frac{1}{2} \alpha \eta C_g$\\
    (a) $\eta := \argmin \{ \eta \in [0,1] \mid  g( x_s + \eta( u - x_s ) )\}$\\
    (b-i) $x_{s+1} :=x_s + \eta (u_s - x_s)$\\
    (b-ii) $x_{s+1} := \argmin \{ x \in \conv(u_0, \ldots, u_{s}) \mid g(x) \}$\\
    }
    \caption{Conditional Gradient (CG) \citep{Jaggi2013}}\label{alg:cg}
    \end{algorithm}
\end{wrapfigure}

\paragraph{Conditional Gradient.}
The CG method is a.k.a. the Frank-Wolfe algorithm.
Let $\dom \subset \X$ be a closed convex bounded subset, $g \in C^1(\dom)$ be a convex function, and $C_g := \sup \{ \frac{2}{\eta^2} ( g(x')  - g(x) - (x'-x) \cdot \nabla g(x) ) \mid x,u \in \dom, \eta \in [0,1], x' = x + \eta (u-x)\}$ be the \emph{curvature constant}. If $g$ is $L$-smooth, i.e. $\nabla g$ is $L$-Lipschitz, then we have an estimate: $C_g \le \mathsf{diam}(\dom)^2 L$.
\refalg{cg} lists a typical CG method with three options in approximation quality $\alpha \ge 0$, line-search (a) and updating scheme (b-i/ii). The classical CG method corresponds to set $\alpha = 0$, to skip (a) and to select (b-i).
\citet[Theorem~1]{Jaggi2013} states that \refalg{cg} (with any options) achieves 
    $g(x_t) - g(x^*) \le \frac{2 C_g}{t+2}(1+\alpha).$
Here $x^* \in \dom$ is an optimal solution. We remark (1) that the fast convergence of KQ methods relies on the fast convergence of CG, and (2) that a CG method can achieve linear convergence $O(e^{-t})$ when the optimizer $x^*$ is in the strict interior \citep{Lacoste-Julien2015,Beck2004}. We refer to \citet{Jaggi2013} and \citet{Beck2004} for more details on CG.

\subsection{KME and MMD for Parameter Distributions} \label{sec:metric}
In a similar fashion to those for probability distributions, starting from a kernel function $K$ on the parameter space $\V$, we define the KME and MMD for parameter distributions $\mu \in \M$.
Rigorous proofs are given in \refsupp{kme.mmd}. 

Let $K : \V \times \V \to \CC$ be a measurable positive definite kernel,
and $\H$ be its associated reproducing kernel Hilbert space (RKHS) equipped with inner product $\iprod{\wdot, \wdot}_K$.
We often use a corollary of the reproducing property: $\| K(\wdot, v) \|_K = \sqrt{K(v,v)}$ for every $v \in \V$.

Let $\M_K$ be the collection of $\mu \in \M$ that satisfies the \emph{kernel moment condition}:
    $\int_{\V} \| K(\cdot, v) \|_K \dd |\mu|(v) < \infty.$
We need this to show the well-definedness of mean elements (i.e. the reproducing property of the expectation).

We define the \emph{kernel mean embedding (KME)} of $\mu \in \M_K$ by $K$, with a slight abuse of notation, as
    $\mu \mapsto K[\mu] := \int_{\V} K(v,\wdot) \dd \mu(v).$
Here, the integral is in the sense of the Bochner integral.
The moment conditions guarantees that $K[\mu]$ always belongs to $\H$.
When the KME is injective, we say $K$ is \emph{characteristic}, and KME $K : \M_K \to \H$ is literally an embedding.
Every mean element $K[\mu]$ enjoys the \emph{reproducing property of the `expectation'}:
    $\iprod{h, K[\mu]}_K = \mu[h]$ for any $h \in \H$.
Here, we wrote $\mu[h] := \int_\V h(v) \dd \mu(v)$.

We define the \emph{maximum mean discrepancy (MMD)} of $(\mu,\nu) \in \M_K \times \M_K$ as an \emph{integral `probability' metric (IPM)}:
    $\mmd[ \mu, \nu ]
    := \sup_{ \| h \|_K \le 1 } \left\{ \mu[h] - \nu[h] \right\}.$
By the reproducing property and the Cauchy-Schwarz inequality, it has another expression:
    $\mmd[ \mu, \nu ]
    = \sup_{ \| h \|_K \le 1 } \iprod{ h, K[\mu] - K[\nu] }_K
    = \| K[\mu] - K[\nu] \|_K.$
If $K$ is characteristic, then the MMD is strictly a distance on $\M_K$.

\subsection{Characterizations of the KME and MMD for Parameter Distributions} \label{sec:power}
Intuitively speaking, the mean element $K[\mu] \in \H_K$ provides an infinite-dimensional coordinate of $\mu \in \M_K$ because the Hilbert space $\H_K$ is an `infinite-dimensional' vector space ($\RR^\infty$).
At the same time, we can identify the mean element $K[\mu]$ as a `linear functional' $\mu \in \H_K^*$ because 
the MMD is an \emph{integral `probability' metric (IPM)} \citep{Muller1997}, or the dual norm $\| \mu \|_{K^*} = \| K[\mu] \|_K$.
In this perspective, the KME is an isometric embedding $\M_K \hookrightarrow \H_K^*$.
We refer to \citet{Simon-Gabriel2018} and \citet{Ambrosio2018} for similar developments in the KMEs for Schwarz distributions and geometric objects.

In general, the KME and MMD for parameter distributions is stronger than the induced norm $\| \mu \|_\F$ and weaker than the TV norm $\| \mu \|_{TV}$.

\begin{thm} \label{thm:relations}
(i) If $K$ and $\varphi$ satisfy $C_{K \to \F} := \| \varphi \|_{\F \otimes K} = \| \| \varphi(x; \cdot ) \|_K \|_\F < \infty$,
then
$\| S[\mu] \|_\F \le C_{K \to \F} \| K[\mu] \|_K$; 
(ii) If $K_1$ and $K_2$ satisfy $K_1 \ll \gamma^2 K_2$ for some $\gamma > 0$, then 
$\| K_1[\mu] \|_{K_1} \le \gamma \| K_2[\mu] \|_{K_2}$; and 
(iii) If $K$ satisfies $C_{TV \to K} := \| K \|_{K \otimes L^\infty} = \sup_{v \in \V} \sqrt{K(v,v)} < \infty$,
then
$\| K[\mu] \|_{K} \le C_{TV \to K} \| \mu \|_{TV}.$ 
\end{thm}
See \refsupp{proof.relation} for the proofs.
Putting all the relations together, 
we can say that MMD interpolates $\M_{TV}$ and $\M_\F$:
\begin{align}
    \M_{TV} \hookrightarrow \M_{K_2} \hookrightarrow \M_{K_1} \hookrightarrow \M_\F.
\end{align}

We supplement further solid characterizations.
(1) Recall that KME is an embedding $\M_K \hookrightarrow \H_K$. Thus, the Rademacher complexity is estimated by $\rad_n(\M_K) \le \rad_n(\H_K)$.
(2) Recall that the Sobolev space $H^s(\RR^d) (s > 0)$ is an RKHS with the Matérn kernel.
Thus, if we take $H_K = H^s$, then we have $\M_K \hookrightarrow H^{-s}$ (the dual Sobolev space).

\subsection{Kernel Quadrature for Parameter Distributions} \label{sec:gkq}

\begin{wrapfigure}[15]{R}{0.5\textwidth}
    \vspace{-0.45cm}
    \begin{algorithm}[H]
    \SetKwInOut{Input}{input}\SetKwInOut{Output}{output}
    \Input{initial $h_0 \in \dom(\lambda)$, quality $\alpha \ge 0$}
    Let $s_0 := h_0$\\
    \For{$j=1$ \KwTo $p$}{
    Let $\eta_j := 2 / (2+j)$\\
    Find $s_j \in \dom(\lambda)$ s.t. 
    $\iprod{s_j, h_j - K[\mu^o]} \le \min_{s \in \dom(\lambda)} \iprod{s, h_j - K[\mu^o]} + \frac{1}{2}\alpha \eta \lambda^2$\\
    (a) $\eta := \min\left\{ \frac{\iprod{ h_j - K[\mu^o], h_j - s_j }_K}{\| h_j - s_j \|_K^2}, 1 \right\}$\\
    (b-i) $h_{j+1} := h_j + \eta (s_j - h_j)$\\
    (b-ii) $h_{j+1} := \argmin_{h \in \conv(s_0, \ldots, s_{i+1})} J(h)$\\
    }
    \Output{model $h_p = \sum_{j=1}^p w_j K(\cdot, v_j)$}
    \caption{GKQ}\label{alg:cgH}
    \end{algorithm}
\end{wrapfigure}
Let $\M(\lambda) := \{ \mu \in \M_K \mid \| K[\mu] \|_K \le \lambda \}$ and $\dom(\lambda) := \{ K[\mu] \mid \mu \in \M(\lambda) \}$.
We formulate the \emph{general kernel quadrature (GKQ)}
as the CG method in $\dom(\lambda)$ for minimizing the MMD $\frac{1}{2} \| h - K[\mu^o] \|_K^2$
between an estimator $h = K[\mu]$ and the truth $h^o := K[\mu^o] \in \dom(\lambda)$.
Recall that by \refthm{relations}, as long as the constant $C_{K \to \F} = \| \varphi \|_{\F \otimes K}$ exists, (1) the convergence in MMD: $\| K[\mu_p] - K[\mu^o] \|_K \to 0$ always implies the convergence in $\F$: $\| S[\mu_p] - S[\mu^o] \|_{\F} \to 0$, which motivates us to minimizing MMD; and (2) the restriction $\| K[ \mu ] \|_K \le \lambda$ is translated as $\| S[\mu] \|_{\F} \le C_{K \to L^2} \lambda$.

We apply CG as listed in \refalg{cgH}.
We remark (1) that the curvature constant is given by $C_J = \lambda^2$ because the $J$ is $1$-smooth and the diameter of $\dom(\lambda)$ is $\lambda$,
and (2) that if $\lambda = \| K[\mu^o] \|_K + \delta$ for some $\delta > 0$, or $\mu^o$ is in the strict interior of $\dom(\lambda)$, then GKQ achieves a faster convergence.

\begin{thm}[Approximation rates of GKQ] \label{thm:approx}
After $p$ iterations of \refalg{cgH}, the loss converges at
    $\frac{1}{2}\| h_p - K[\mu^o] \|_K^2 \le \frac{2 \lambda^2}{p+2}.$
Furthermore, if $\lambda = \| K[\mu^o] \|_K + \delta$ with some $\delta>0$, then 
$\frac{1}{2}\| h_p - K[\mu^o] \|_K^2 \le
\frac{4 \lambda^4}{p^2 \delta^2}$ without (a), and
$\le 4 \lambda^2 \exp\left( - \frac{p \delta^2}{4 \lambda^2} \right)$ with (a).
\end{thm}
All the three rates have two notable characteristics: They are (1) \emph{independent of the dimension $d$} and (2) controlled by a new complexity measure $\| K[\mu^o] \|_K$.
The first rate $O(1/p)$ reproduces the same rate as the traditional \emph{Maurey-Jones-Barron (MJB) bound} when $\F = L^2(\PP)$ (see \refprop{kurkova}),
or the \emph{uniform law of large numbers (ULLN}) when $\F = L^\infty(\X)$ (see \refthm{ulln.sir});
and the rest two rates $O(1/p^2)$ and $O(e^{-p})$ are new faster results.
The dimension-independence is one of the most notable characteristics of the MJB theory.
For example, another traditional rate such as Jackson-type ratio $O(p^{-s/m})$ \citep{DeVore1993} depends on both input dimension $m$ and smoothness $p$. This is due to the Jackson theory is based on the fixed basis such as Fourier and Taylor series expansion.
We have succeeded to inherit the spirit of MJB by establishing a new constructive proof scheme (\refalg{cgH}).
Moreover, the complexity measure $\| K[\mu^o] \|_K$ is sharper than the original complexity measure $\| \mu^o \|_{TV}$ in the MJB bound
(because \refthm{relations} claims that $\| K[\mu] \|_K \lesssim \| \mu \|_{TV}$.)

\section{Empirical and Unitary Kernel Quadrature} \label{sec:implementation}
We investigate several implementation schemes, EKQ and UKQ, for GKQ. In order to conduct GKQ in practice, we need an estimator of the KME $K[\mu^o]$.
In this section, we assume that the range $\F$ of $S$ to be either an RKHS with kernel $k : \X \times \X \to \CC$ or $L^2(\PP)$ with `kernel' $k(x,y) := \delta(x-y)$.

\subsection{GKQ with Empirical KME}
To begin with, we suppose that we know a pseudo-inverse $S^\dag$ and consider empirically approximating it.
A pseudo-inverse $S^\dag : \F \to \M$ of $S : \M \to \F$ is (not necessarily unique but) any operator that satisfies $S[ S^\dag[f] ] = f$ for any $f \in \F$.
In other words, if $\mu_f = S^\dag[f]$ then $ S[\mu_f] = f$ for any $f \in \F$.
We consider empirically estimating $S^\dag[f]$.

The following proposition reveals the connection between a pseudo-inverse and a reproducing kernel.
\begin{prop} \label{prop:dagphi}
Suppose that $\F$ is an RKHS with kernel $k : \X \times \X \to \RR$, and assume that there exists a function $\dagphi : \X \times \V \to \CC$ s.t.
    $k(x,y) = \int_\V \varphi(x;v) \dagphi(y;v) \dd Q(v)$ for $x,y \in \X$.
Then, 
$S^\dag[f](v) := \iprod{f,\dagphi(\cdot ; v) \dd Q}_\F$ for $f \in \F$ and $v \in \V$
is a pseudo-inverse operator of $S : \M \to \F$.
\end{prop}
See \refsupp{proof.dagphi} for the proof.
For example, consider the Fourier transform, where $\varphi(x;\xi) = \exp(-i x \cdot \xi)$. 
This is a fundamental but extreme case because $\dagphi(x;\xi) = \exp(i x \cdot \xi)$ and $k(x,y) = \delta(x-y)$.
In addition, consider a neural network with $\varphi(x;a,b) = \sigma(a \cdot x - b)$. 
This is less familiar but an interesting case because the $\dagphi$ is known:
$\dagphi(x;a,b) = \rho(a \cdot x - b)$ with an admissible function $\rho$ (see \refsupp{eg.injective} for more details).

\paragraph{EKQ.}
Suppose that $\F$ is an RKHS with kernel $k$ (on $\X$) and we know the $\dagphi$ of pseudo-inverse $S^\dag$.
Given a dataset $D_n = \{ (x_i, y_i)\}_{i=1}^n$, and consider estimating the KME $K[ S^\dag[f^o] ] = K[ \mu^o ]$ with a certain kernel $K$ (on $\V$). Then, 
\begin{align}
\widehat{K[S^\dag[f]]}(v) := \sum_{i=1}^n \alpha_i K[\dagphi]( x_i, v ) \label{eq:emp.Kmu}
\end{align}
is an empirical estimator of the $K[\mu^o]$. Here, the coefficients $\alpha_i$ are set to satisfy $\sum_{j=1}^n \alpha_i k(x_i,x_j) = y_i$. This `empirical estimator' is derived by replacing $f$ in $S^\dag[f]$ with the empirical estimator $\fhat = \sum_{i=1}^n \alpha_i k( \cdot, x_i )$. By \refprop{dagphi}, $S^\dag[\fhat](v) = \langle \fhat, \dagphi(\cdot ; v) \rangle_\F = \sum_{i=1}^n \alpha_i \dagphi( x_i; v )$, which yields $K[ S^\dag[\fhat] ](v) = \sum_{i=1}^n \alpha_i K[\dagphi]( x_i; v )$.

Combined with the empirical KME \refeq{emp.Kmu}, we can substantiate the GKQ \refalg{cgH}.
For the kernels, we can use versatile kernels such as the Gaussian kernel \citep{Scholkopf2001} and the Stein kernel \citep{Chen2018}.
We call this scheme, the GKQ with empirical KME, an EKQ for short. The EKQ has relation to the empirical risk minimization (ERM).
\begin{prop} \label{prop:erm}
Assume $\M_K \hookrightarrow \M_\F$, then the EKQ yields an ERM in $\F$.
\end{prop}
\begin{proof}
By the continuity, $\| K[\mu] - K[\widehat{\mu_f}] \|_K \gtrsim \| S[\mu] - S[\widehat{\mu_f}] \|_\F = \| S[\mu] - \fhat \|_\F$; and by the Cauchy-Schwartz, $\| S[\mu] - \fhat \|_\F \ge \| k_i \|_k^{-1} \iprod{ S[\mu] - \fhat, k_i }_\F = |S[\mu](x_i) - y_i|$ for any $i \in [n]$. \qedhere
\end{proof}

\subsection{GKQ with Unitary Kernel}
However, in many cases, we do not know $S^\dag$ nor $\dagphi$. To circumvent this problem, we consider mapping $K[\dagphi(x;\cdot)](v) = \varphi(x;v)$ (or something else we can compute, if any) by appropriately constructing the kernel $K$.
We define the \emph{standard $\F$-unitary kernel}:
\begin{align}
    U(v,v') := \iprod{\varphi(\cdot,v), \varphi(\cdot,v')}_\F.
\end{align}
By construction, $U$ is a positive definite kernel. 
\begin{prop} \label{prop:kmeU}
For $\dagphi$ in \refprop{dagphi} and the standard $\F$-unitary kernel $U$, we have 
    $U[ \dagphi(x ; \cdot) ](v) = \varphi(x;v)$, and
    $U[ \mu ](v) = \iprod{ S[\mu], \varphi(\cdot;v) }_\F$.
Furthermore, $U$ induces the topology of $\F$: $\iprod{ U[\mu], U[\nu] }_U = \iprod{ S[\mu], S[\nu] }_\F$ for every $\mu, \nu \in \M_U$.
\end{prop}
See \refsupp{proof.kmeU} for the proof.
As a consequence, however difficult to find a pseudo-inverse $S^\dag[f]$, its KME $U[ S^\dag[f] ]$ is easily computed in $\F$.
The term `unitary' comes from the following property that $U \circ S^\dag : \F \hookrightarrow \H_U$ becomes a unitary isomorphism:
    $\mmd[ \mu, \nu ] = \| U[\mu] - U[\nu] \|_U = \| S[\mu] - S[\nu] \|_\F.$ 
The isometry means that $\M_U$ is the largest space that satisfies $\M_K \hookrightarrow \M_\F$.

\paragraph{Empirical UKQ.}
\newcommand{\Vhat}{\widehat{\V}}
\begin{wrapfigure}[15]{R}{0.5\textwidth}
    \vspace{-0.5cm}
    \begin{algorithm}[H]
    \SetKwInOut{Input}{input}\SetKwInOut{Output}{output}
    \Input{dataset $\{ (x_i, y_i) \}_{i=1}^n$, proposal distribution $\QQ$ on $\V$, candidates number $q$, unit number $p$}
    Let $v_1 \sim \QQ$
    \For{$j=2$ \KwTo $p$}{
    Draw $q$ candidates $\Vhat := \{ v_t \}_{t=1}^q \sim \QQ$\\
    Find $v_j := 
    \argmax_{v \in \Vhat} \left[ 
    \sum_{j=1}^{p-1} \sum_{i=1}^n \varphi(x_i;v_j) \varphi(x_i;v) \right.$ \\
    $\left. +\frac{1}{2}\sum_{i=1}^n |\varphi(x_i;v)|^2 -j\sum_{i=1}^n y_i \varphi(x_i;v) \right]$
    }
    Let $w_j := 1/p$ for all $j \in [p]$.\\
    \Output{model $S[\muhat] := \sum_{j=1}^p w_j \varphi(\cdot ; v_j)$}
    \caption{GKQ in practice}\label{alg:ukq}
    \end{algorithm}
\end{wrapfigure}

For a concrete example, we list in \refalg{ukq} the empirical UKQ for $\F = L^2(\PP)$ with empirical unitary kernel $U(v,v') = \frac{1}{n}\sum_{i=1}^n \varphi(x_i;v) \varphi(x_i;v')$.
In order to avoid the non-convex subproblem in $v$, we draw sufficiently many, say $q$, candidates $v$ from a proposal distribution $\QQ = N(0,I)$, and select from the candidates one by one.
In general, the so-called Frank-Wolfe step, which orders to find $s_j \in \dom(\lambda)$ in \refalg{cgH}, is not always tractable.
Therefore, we employ the \emph{greedy minimization} for a practical implementation following \citet{Chen2010}.
We simply start from an arbitrary $v_1 \in \V$, and let $w_j = 1/p$ at the $p$-th iteration.
In the $p$-th iteration, we select $v$ that maximizes the difference $\mmd[ \mu, \mu_p ]^2 - \mmd[  \mu, \mu_{p-1} ]^2$,
or equivalently, 
    $v_p := \argmax_{v \in \V} \left[ \frac{1}{2}K(v, v) + \sum_{j=1}^{p-1} K(v_j, v) - p K[\mu](v)\right].$
This is also proved to converge at $O(1/p^2)$ \citep{Chen2010}.

If $S$ is not injective, then $U$ cannot be characteristic because the non-injectivity means that there exists a non-trivial null space $\ker S$, and $S[\mu + \ker \mu] = S[\mu]$ for any $\mu$, which implies $U[\mu + \ker \mu] = U[\mu]$ by the unitarity.
The non-trivial kernel exists, for example, when $S$ is a neural networks. See \refsupp{eg.injective} for more details.
Nevertheless, the unitary kernel may be a natural choice for the KME of parameter distributions because 
(1) 
the KME induces the geometry of $\F$ with $\M_U$ up to the inner product,
and (2) it kills the difference between $\mu$ and $\mu + \ker S$. In other words, even though the determination of $\ker S$ and $S^\dag$ is usually impossible,
we can handle the quotient space $\M_U / \ker S$ as a Hilbert space with inner product $\iprod{ [\mu], [\nu] } := \iprod{ U[\mu], U[\nu] }_U = \iprod{ S[\mu], S[\nu] }_\F$.
In theory, there are infinitely many different unitary kernels, and the standard unitary kernel is just one computable example for the implementation.

\section{Convergence Analysis}

We consider EKQ, the GKQ with empirical KME, 
or an empirical version of \refalg{cgH} where $K[\mu^o]$ is replaced with $K[\widehat{\mu_f}]$ explained in \refeq{emp.Kmu}.
An EKQ can be reformulated as ERM, and includes GKQ with empirical unitary kernels (UKQ).
The dataset $D_n = \{ (x_i, y_i) \}_{i=1}^n$ is generated from the regression model $Y = S[\mu^o](X) + \xi$ with unknown $\mu^o \in \M_K(\lambda) = \{ \mu \in \M_K \mid \| K[\mu] \|_K \le \lambda \}$ and $\xi \sim N(0,\sigma^2)$.
the final output $\fhat$ is given by $\fhat = S[\muhat]$, and $\muhat$ is given by EKQ.
See \refsupp{proof.gen.err.assums} for more detailed list of assumptions.

\begin{thm} \label{thm:gen.err}
Let $\M_p := \{ \sum_{j=1}^p w_j \delta_{v_j} \mid \sum_{j=1}^p |w_j| \le D_{TV}, |v_j| \le D_\V, \| K_v \|_K \le D_K \}$ and write $R_\infty := \sup_{\mu \in \M_p} \| S[\mu] \|_\infty$.
Assume that $\varphi(x;v)$ is $L_\varphi$-Lipschitz in $v$ at every $x$, and satisfies $C_{K \to \infty} = \| \varphi \|_{L^\infty \otimes K} < \infty$.
Then, after $p$ iterations of EKQ, we have a fast learning rate:
\begin{align}
    \| \fhat - f^o \|_2^2 &
    \lesssim \err^2_{app} + (R_\infty^2 + \sigma^2) \err^2_{est} + \err^2_{res}
\end{align}
with probability at least $1-2e^{-r}$ for any $r>0$.
Here, we defined
$\err^2_{app} := \inf_{\mu^* \in \M_p} \| S[\mu^*] - S[\mu^o]\|_2^2$, 
$\err^2_{est} := \frac{pd}{n}\log_+\left( 1 + \frac{4 L_\varphi \max\{ D_{TV}, D_\V \} \sqrt{n}}{R_\infty \wedge \sigma}\right)$,
and the residual $\err^2_{res} := O_p( \frac{R_\infty^2 + \sigma^2}{n} )$.
\end{thm}

See \refsupp{proof.gen.err} for the proof.
By the assumptions, we can estiamte $R_\infty = \sup_{\mu \in \M(p)} \| S[\mu] \|_\infty = C_{K \to \infty} \min\{ \lambda, D_{TV} D_K \}$. This is notable that $D_K$ and $\| K[\mu] \|_K (\le \lambda)$ control the estimation error. Thus, we can understand that these quantities reflects the quality of embedding $K$.
We remark that the Lipschitz constant $L_\varphi$ can also be replaced with a Lipschitz constant $\| K_v \|_K$ in $v$.

Finally, we consider the model selection problem of parameter number $p$.
Recall that (the squared error) \refthm{approx} states either $\err^2_{app} = O(\lambda^2/p), O(\lambda^4/p^2)$ or $O(\exp(-p/\lambda^2))$.
Combined with the fast learning rate $\err^2_{est} = O_p( (R_\infty^2+\sigma^2) pd/n )$,
we can minimize the excess risk as $\| \fhat - f^o \|_2^2 = \widetilde{O}( R_\infty^2 \lambda^2 d/n )$ by letting $p = O( \lambda^2 \log \frac{n}{\lambda^2 R_\infty^2 d} )$. 

\section{Experimental Results}
\begin{figure*}[t]
\centerline{
\subfigure[SIR]{
\includegraphics[trim=3cm 1cm 3.5cm 2cm, clip, width=0.5\hsize]{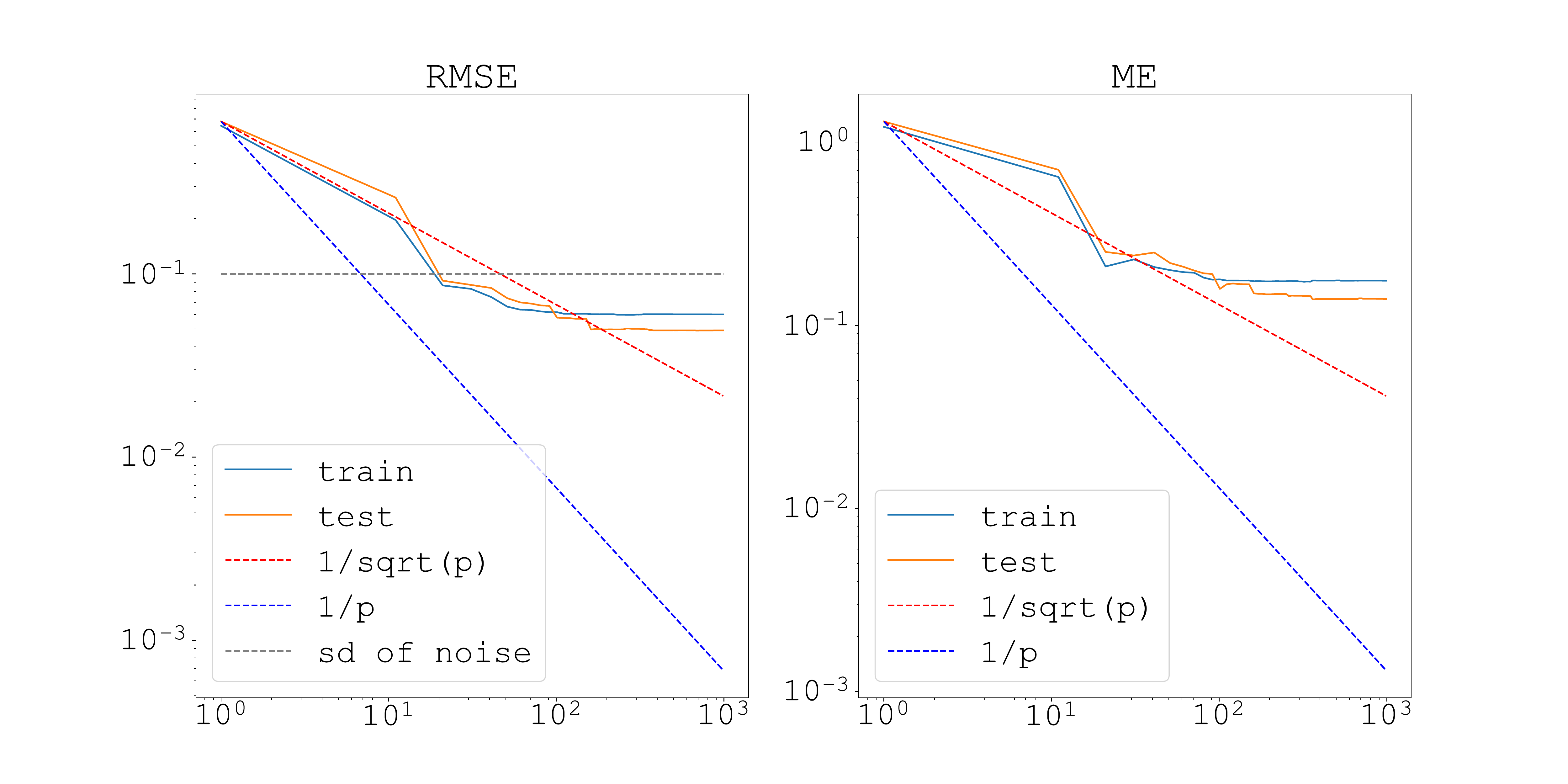}
\label{fig:exp1}
}
\subfigure[UKQ (proposed)]{
\includegraphics[trim=3cm 1cm 3.5cm 2cm, clip, width=0.5\hsize]{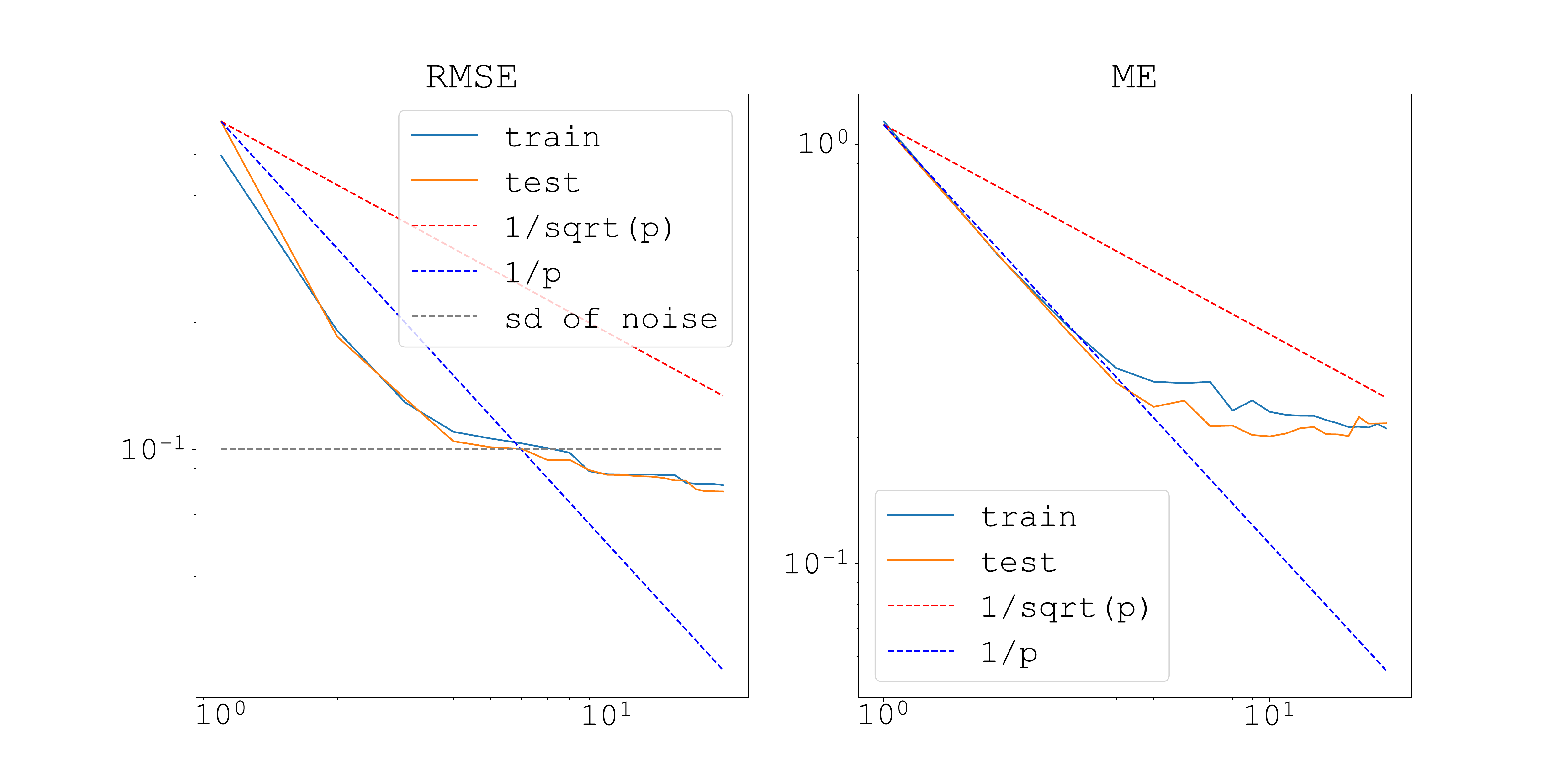}
\label{fig:exp2}
}
}
\caption{Error decay of numerical integration results (double logarithmic plot). Red and Blue lines correspond to $O(1/\sqrt{p})$ and $O(1/p)$ respectively. Gray line in the RMSE plot indicates the standard deviation of the noise in the dataset. Both SIR and UKQ reached Gray line, and UKQ decays faster than SIR.}
\label{fig:exp}
\end{figure*}

Figure 1 compares the error decay of the sampling importance resampling (SIR) and UKQ (proposed). In the comparison of RMSE, both SIR and UKQ reached the gray line, which is the standard deviation of the noise in the dataset. This indicates that the numerical integration was conducted correctly in both cases. SIR decayed along the red line, which corresponds to the slower rate $O(1/\sqrt{p})$.
On the other hand, UKQ decayed along the blue line, which corresponds to the faster rate $O(1/p)$.
See the Supplementary materials for more details.

\section{Conclusion}
    We developed the GKQ for training infinite-width models $S[\mu]$ with signed/complex/vector-valued parameter distributions.
    As overviewed in the Introduction, a variety of infinite-width models are investigated in the theoretical study of neural networks, but their topologies are questionable. The KME and MMD for parameter distributions induce the metric of functions with the space of parameters, 
    and the GKQ provides a fast training method based on the metric. By minimizing the approximation-estimation tradeoff in parameter number $p$, we can achieve the best squared error risk $O( \| K[\mu^o] \|_K^2 d/n)$ with a new complexity measure $\| K[\mu^o] \|_K$. 

\section*{Broader Impact}
We believe this section is not applicable to this paper because of the theoretical nature of this study.

\section*{Acknowledgments}
The author is grateful to Noboru Murata for suggesting the initial idea of the problem.
The author would like to thank Taiji Suzuki, Atsushi Nitanda, Masahiro Ikeda, Isao Ishikawa, Kenji Fukumizu, Motonobu Kanagawa, Toni Karvonen, François-Xavier Briol, Tatsunosuke Shimada, and Takuo Matsubara for useful discussions.
This work was supported by JSPS KAKENHI 18K18113.

\bibliography{summary_library}
\bibliographystyle{abbrvnat}

\newpage
\appendix

{
\newcommand{\domain}{\mathcal{X}}
\newcommand{\vm}{{\bm{\mu}}}
\newcommand{\vn}{{\bm{\nu}}}
\newcommand{\vmeas}{\mathcal{M}}
\newcommand{\prob}{\mathcal{P}}
\newcommand{\mmap}{L}
\newcommand{\fmap}{\varphi}
\newcommand{\gpk}{\mathsf{k}}

\section{Fundamental Properties of KME for Vector Valued Measures} \label{sec:kme.mmd}
We investigate the fundamental properties of the kernel mean embedding (KME) for vector-valued measures, which includes $\RR$- and $\CC$-valued measures.

\subsection{Well-Definedness}
Let $(\domain,\B)$ be a measurable space, $(V,| \wdot |)$ be a Banach space, and $\vmeas$ be a collection of $V$-valued vector measures $\vm$ on $\domain$.
Let $L(\X)$ be the space of $\CC$-valued measurable functions on $\domain$.
$\vm[f(X)]$ denotes the integral $\int_{\domain} f(x) \dd \vm(x)$ for any $f \in L(\X)$.
$|\vm|$ denotes the total variation measure of $\vm$.
We have the following inequality.
\begin{align}
	|\vm[f]| \leq |\vm|[ |f| ].
\end{align}

Let $k : \domain \times \domain \to \CC$ be a measurable positive definite kernel, and $H$ be its associated reproducing kernel Hilbert space (RKHS).
Let $\vmeas_k(\domain;V)$ be the collection of finite $V$-valued vector measures on $\domain$ that satisfies
\begin{align}
	B_1 := |\vm|[\sqrt{k(X,X)}] = |\vm|[\| k(\cdot,X) \|_k] < \infty.
\end{align}

We define the KME for a vector measure $\vm \in \vmeas_k(\domain;V)$ as
\begin{align}
	k[\vm] := \vm[ k(\cdot,X) ] := \int_{\domain} k( x, \wdot ) \dd \vm (x).
\end{align}
For an arbitrary $\vm \in \vmeas_{k}(\domain;V)$, we define a linear functional $\mmap_{\vm} : H \to \CC; f \mapsto \vm[ \iprod{f,k(\cdot,X)}_k]$.
Then, $\mmap_{\vm}$ is bounded because 
\begin{align*}
	| \mmap_{\vm} [f] | = | \vm[ \iprod{f,k(\cdot,X)}_k] | \leq |\vm|[ |\iprod{f,k(\cdot,X)}_k|] \leq \| f \|_k  B_1 < \infty.
\end{align*}
Therefore, $k[\vm] \in H$.
By the Riesz's representation theorem, there exists $k[\vm] \in H$ for any $f \in H$ such that
$\mmap_{\vm}[f] = \iprod{f, k[\vm]}_k$.
In particular, the reproducing property of the `expectation' holds:
\begin{align}
	\vm[f(X)] = \iprod{f, k[\vm]}_k.
\end{align}

\subsection{Characteristics}
A bounded measurable positive definite kernel $k$ on $\domain$ is said to be \emph{characteristic} when
the KME operator $\vmeas(\domain;V) \to H; \vm \mapsto k[\vm]$ is injective. In other words, $k$ is characteristic if and only if
\begin{align}
	\forall f \in H \, : \, \vm[f] = \vn[f] \quad \Rightarrow \quad \vm = \vn.
\end{align}

We claim that if the KME operator with respect to $k$ is injective on the collection $\prob(\domain)$ of probability measures on $\domain$,
then it is also injective on $\vmeas_{TV}(\domain;\RR^d)$.

\begin{proof}
We prove the claim first for $\vmeas_{TV}(\domain;\RR)$, then for $\vmeas_{TV}(\domain;\RR^d)$.

First, we consider the KME for signed measures.
Recall that $L : \vmeas_{TV}(\domain;\RR) \to H; \mu \mapsto k[\mu]$ is a bounded linear operator.
Since $\hull_\RR \prob(\domain) \subset \vmeas_{TV}(\domain;\RR)$, 
the restriction $L|_{\prob(\domain)}$ is injective.
We show the converse: If $L|_{\prob(\domain)}$ is injective, then $L$ is injective.
Assume that there exist $\mu, \nu \in \vmeas_{TV}(\domain;\RR)$ that satisfies $L[\mu] = L[\nu]$.
Since $\hull_\RR \prob(\domain) = \vmeas_{TV}(\domain;\RR)$
(because $\hull_\RR \prob(\domain) \subset \vmeas_{TV}(\domain;\RR)$ is obvious and the converse follows from Hahn's decomposition theorem),
we can take a basis $\{ p_i \}_{i=1}^d$ of $\vmeas_{TV}(\domain;\RR)$ that is composed of the elements of $\prob(\domain)$.
Thus, by rewriting $\mu = \sum_i^d \mu_i p_i$ and  $\nu = \sum_{i}^d \nu_i p_i$, we have $\sum_{i=1}^d \mu_i L[p_i] = \sum_{i=1}^d \nu_i L[p_i]$.
By the assumption that $L|_{\prob(\domain)}$ is injective, the image set $\{ L[p_i] \}$ is linearly independent, which concludes $\mu_i = \nu_i$ for every components. Namely, $\mu = \nu$.

Then, we consider the KME for vector measures.
Assume that $\vm, \vn \in \vmeas_{TV}(\domain;\RR^d)$ satisfies $k[\vm] = k[\vn]$.
In other words, $k[\mu_i] = k[\nu_i]$ for every component $\mu_i$ and $\nu_i (i=1 \ldots, d)$.
In each component, we can reuse the result for the KME for signed measures, and obtain $\mu_i = \nu_i, \ (i=1, \ldots, d)$,
which concludes $\vm = \vn$.
\end{proof}
}

\section{Examples of Non-injectivity} \label{sec:eg.injective}
In general, the operator $S$ is not always injective and there often exist two different distributions $\mu$ and $\nu$ that indicate the same function: $\mu \neq \nu$ but $S[\mu] = S[\nu]$.
In other words, $S$ often has a non-trivial null space $\ker S \neq \emptyset$.

When $S$ is the {\bf Fourier transform} $S[ \mu ](x) = \int_\RR e^{-i \xi  x} \dd \mu(\xi)$, and $\mu$ admits an $L^2$-density function,
then it is \emph{injective} because the Fourier transform is a bijection $L^2(\RR) \to L^2(\RR)$.
The right inverse $R$ is given by the inverse Fourier transform $R[f](\xi) = (2 \pi)^{-1} \int_\RR e^{i \xi x} f(x) \dd x$.

When $S$ is the {\bf shallow neural network} $S[\mu](x) := \int_{\RR^m \times \RR} \sigma( a \cdot x - b ) \dd \mu(a,b)$ with activation function $\sigma : \RR \to \CC$, then it is \emph{not injective}.
To our surprise, a right inverse $R$ is explicitly given by the \emph{ridgelet transform} \citep{Murata1996,Candes.PhD,Sonoda2015}:
$R[f;\rho](a,b) := \int_{\RR^m} f(x) \overline{\rho(a \cdot x - b)} \dd a \dd b$
with a Schwartz function $\rho : \RR \to \CC$. When $\rho$ satisfies the admissibility condition $\int_{-\infty}^\infty \widehat{\sigma}(\zeta) \overline{\widehat{\rho}(\zeta)} |\zeta|^{-m} \dd \zeta = 1$, the reconstruction formula $S[R[f;\rho]] = f$ holds for any $f \in L^1(\RR^m)$ \citep[Theorem~5.6]{Sonoda2015}.
It is not difficult to find two different $\rho_1$ and $\rho_2$ that satisfy the admissibility condition, which clearly suggests the non-injectivity because $\mu_1 := R[f;\rho_1]$ and $\mu_2 := R[f;\rho_2]$ satisfy $\mu_1 \neq \mu_2$ but $S[\mu_1] = S[\mu_2] = f$.

\section{Barron's bound}

\begin{prop}[Barron's bound, Corollary~5.4(i), \citet{Kurkova2012}] \label{prop:kurkova}
    Let $L^2(\PP)$ be a Hilbert space on $\X \subset \RR^m$ with $\sigma$-finite measure $\PP$,
    and $\M$ be a space of $\CC$-valued Borel measures on $\V \subset \RR^d$ equipped with the total variance norm $\|\mu\|_{TV} := |\mu|(\V)$ for every $\mu \in \M$. Let $\varphi : \X \times \V \to \RR$ be a map, and assume that $G := \{ \varphi( \cdot, v ) \mid v \in \V \}$ be a bounded subset in $L^2(\PP)$ with bound $C_{2 \to \infty} := \sup_{g \in G} \| g \|_{L^2(\PP)}$. For any $p \in \NN$ and $f := S[\mu]:= \int_{\V} \varphi(x;v) \dd \mu(v)$, which is always bounded because $\| S[\mu] \|_{L^2(\PP)} \le C_{2 \to \infty} \| \mu \|_{TV}$, we have
    \begin{align}
        \| f - \hull_p G \|_{L^2(\PP)} \le \sqrt{\frac{C_{2 \to \infty}^2 \| \mu \|_{TV}^2 - \| f \|_{L^2(\PP)}^2}{p}}.
    \end{align}
\end{prop}

\section{Convergence rate of Monte-Carlo method}

\begin{prop}[Uniform law of large numbers for Monte-Carlo integration] \label{prop:ulln}
    Let $k : \X \times \X \to \RR$ be a measurable positive definite kernel on $\X$ and $\F$ be the RKHS with $k$.
    Let $X \sim \PP,$ $D_n := \{ X_i \}_{i=1}^n \sim \PP$ be i.i.d. $n$-copies of $X$.
    Assume that the bound $C_{L^2 \to k}^2 := \| k \|_{k \otimes L^2(\PP)}^2 = \EE k(X,X)$ exists.
    Then, for any $B>0$,
    \begin{align}
        \EE_{D_n}\left[ \sup_{\| f \|_k \le B} \Bigg| \frac{1}{n}\sum_{i=1}^n f(X_i) - \EE[f(X)] \Bigg| \right] \le \frac{4BC_{L^2 \to k}}{\sqrt{n}}.
    \end{align}
    \begin{proof}
    We refer to \citet[\S~8.2.1]{Vershynin2018} for basic ideas.
    Write $\F(B) := \{ f \in \F \mid \| f \|_k \le B \}$.
    By the symmetrizaion argument \refprop{sym}, the LHS is bounded by $2 \rad_n(\F(B))$;
    and by the standard Rademacher calculus, $\rad_n( \F(B) ) \le 2 B \sqrt{ \EE k(X,X) / n}$ \citep[\S~4.3]{Bartlett2002}, which yields the claim.
    \end{proof}
    
    Now, consider a sampling based discretization of the integral representation $S[\mu]$.
    \begin{thm} \label{thm:ulln.sir}
    Let $K : \V \times \V \to \RR$ be a measurable positive definite kernel on $\V$ and $\H$ be the RKHS with $K$.
    Let $\mu$ be a finite Borel measure on $\V$ and write $\PP_\mu := |\mu|/\| \mu \|_{TV}$.
    Let $\V_p := \{ v_j \}_{j=1}^p \sim \PP_\mu$ i.i.d., and write $w_j := \mu(v_j)/|\mu(v_j)|$.
    Assume that the bound $C_{L^2 \to K}^2 := \| K \|_{K \otimes L^2(\PP_\mu)}^2 = \int_\V K(v,v) \dd \PP_\mu$ exists;
    and that there exists $B > 0$ such that $\| \varphi(x ; \cdot ) \|_K \le B$ for any $x \in \X$. Then,
    \begin{align}
        \EE_{\V_p} \left[ \sup_{x \in \X} \Bigg| \sum_{j=1}^p w_j \varphi(x;v_j) - S[\mu](x) \Bigg|\right] \le \frac{4 B C_{L^2 \to K} \| \mu \|_{TV}}{\sqrt{p}}
    \end{align}
    \end{thm}
\end{prop}

\section{Proofs}

\subsection{\refthm{relations}} \label{sec:proof.relation}
\begin{proof}
    (i) By the reproducing property and the Cauchy-Schwartz inequality, $| S[\mu](x) | = | \iprod{ \varphi(x,\cdot), K[\mu] }_K | \le \| \varphi(x;\cdot)\|_K \| K[\mu] \|_K$ for any $\mu \in \M_K$ and $x \in \X$. Taking the norm $\| \cdot \|_\F$ yields the assertion.\\
    (ii) $K_1 \ll \gamma^2 K_2$ for some $\gamma >0$ is equivalent to $\| h \|_{K_2} \le \gamma \| h \|_{K_1}$ for any $h \in \H_1$. Thus, $\| K_1[\mu] \|_{K_1} = \sup_{\| h \|_1 \le 1} | \mu[h] | \le \sup_{\| h/\gamma \|_2 \le 1} | \gamma \mu[h/\gamma] | = \gamma \| K_2[\mu] \|_{K_2}$.\\
    (iii) $\| K[\mu] \|_K^2 = \int_{\V \times \V} K(v,v') \dd \mu(v) \otimes \mu(v') \le \left\{ \int \sqrt{K(v,v)} \dd \mu \right\}^2 \le C_{TV \to K}^2 \| \mu \|_{TV}^2$.
\end{proof}

\subsection{\refprop{dagphi}} \label{sec:proof.dagphi}
\begin{proof}
$S[S^\dag[f]](x) = \iprod{ f, \int_\V \varphi(\cdot;v) \dagphi(x;v) \dd v }_\F = \iprod{f,k(\cdot,x)}_\F = f(x)$.
\end{proof}

\subsection{\refprop{kmeU}} \label{sec:proof.kmeU}
\begin{proof} 
By changing the order,
$U[ \dagphi(x ; \cdot) ](v) = \iprod{ \varphi(\cdot; v), k(\cdot, x) }_\F = \varphi(x;v)$,
and
$U[\mu](v) = \int_{\V} \iprod{ \varphi(\cdot; v), \varphi(\cdot; v') }_\F \dd \mu(v') = \iprod{ \varphi(\cdot; v), S[\mu] }_\F.$
For the unitarity, 
it is sufficient to show the isometry: $\| U [\mu] \|_U = \| S[\mu] \|_\F$ for any $\mu \in \M_U$. 
By using the reproducing property of expectation, and the KME by $U$, we have
$\| U[\mu] \|_U^2
    = \iprod{ U[\mu], U[\mu] }_U = \mu[ U[\mu] ] 
    = \int_{\V \times \V} \iprod{ \varphi(\cdot; v), \varphi(\cdot; v') }_\F \dd \mu(v) \otimes \mu(v')
    = \iprod{ S[\mu], S[\mu] }_\F = \| S[\mu] \|_\F^2.$ \qedhere
\end{proof}

{
\newcommand{\out}{\mathsf{output}}
\newcommand{\ev}{\mathsf{v}}
\newcommand{\Fp}{\mathcal{F}_p}
\newcommand{\Fo}{\mathcal{F}^o}
\newcommand{\Ro}{\dot{R}}
\newcommand{\Cgc}{C_1}
\newcommand{\Crc}[1]{C_{2#1}}
\newcommand{\Clog}{C_3}

\section{Fundamental Tools}

We say a collection $\G$ of measurable functions $g : \Z \to \RR$ is a \emph{(separable or complete) Carathéodory set} when
there exist a (separable or complete) metric space $(T,d)$ and a surjection $T \to \G$ such that $T \ni t \mapsto g_t(z)$ is continuous at every $z$.

\begin{prop}[{Talagrand's inequality \citep[Theorem~7.5]{Steinwart2008}}] \label{prop:talagrand}
Let $(\Z,\PP)$ be a probability space, and let $\G \subset L(Z)$ be a separable Carathéodory set.
Assume that (i) $\PP g = 0$ for any $g \in \G$, (ii) $\sup_{g \in \G} \PP g^2 = R^2$ for some $R^2 >0$, and (iii) $\sup_{g \in \G} \| g \|_\infty = B$ for some $B >0$.
Then, for any $\gamma>0$ and $t>0$, we have
\begin{align}
    \PP^n \left\{ z \in \Z^n \,\Bigg|\, \sup_{g \in \G} \Big| \frac{1}{n} \sum_{i=1}^n g(z_i) \Big| \ge (1+\gamma) \PP^n\left[ \sup_{g \in \G} \Big| \frac{1}{n} \sum_{i=1}^n g(z_i) \Big| \right] + \sqrt{\frac{2 t R^2}{n}} + \left( \frac{2}{3}+\frac{1}{\gamma}\right)\frac{t B}{n} \right\} \le e^{-t}.
\end{align}
\end{prop}

\begin{prop}[{Peeling \citep[Theorem~7.7]{Steinwart2008}}] \label{prop:peeling}
Let $(\Z,\PP)$ be a probability space, $(T,d)$ be a separable metric space, and $\G \subset L(Z)$ be a separable Carathéodory set with $(T,d)$. Let $h : T \to [0,\infty)$ be a continuous function, and define $r^* = \inf_{t \in T} h(t)$. Let $\phi : (r^*,\infty) \to [0,\infty)$ be a function such that (i) $\phi(4r) \le 2 \phi(r)$ and (ii) $\PP[ \sup_{t \in T, h(t) \le r} |g_t| ] \le \phi(r)$ for any $r \ge r^*$. Then, we have
\begin{align}
    \PP\left[ \sup_{t \in T} \frac{g_t}{h(t) + r} \right] \le \frac{4 \phi(r)}{r}.
\end{align}
\end{prop}

\begin{dfn}[Rademacher and Gaussian Complexities]
Let $(\Z,\PP)$ be a probability space, $g : \Z \to \RR$ be a measurable function, and $\G \subset L(\Z)$ a collection of measurable functions.

We write $R_n(f) := \frac{1}{n} \sum_{i=1}^n \sigma_i f(x_i)$ with a Rademacher sequence $\{ \sigma_i = \pm 1 \}_{i=1}^n$, and $R_n(\F) := \sup_{f \in \F} R_n(f)$.
We call $\radhat_n(\F) := \EE_{\sigma}[ R_n(\F) ] = \EE_{\sigma}[ \sup_{f \in \F} R_n(f) ]$ the \emph{empirical Rademacher averages}, and $\rad_n(\F) := \EE_{D_n}[ \radhat(\F) ] = \EE[ R_n(\F) ]$ the \emph{Rademacher average} or the \emph{Rademacher complexity}.

Similarly, we write $G_n(f) := \frac{1}{n} \sum_{i=1}^n \xi_i f(x_i)$ with a Gaussian sequence $\{ \xi_i \}_{i=1}^n \sim N(0,1)$, and $G_n(\F) := \sup_{f \in \F} G_n(f)$.
We call $\gsshat_n(\F) := \EE_{\xi}[ G_n(\F) ] = \EE_{\xi}[ \sup_{f \in \F} G_n(f) ]$ the \emph{empirical Gaussian averages}, and $\gss_n(\F) := \EE_{D_n}[ \gsshat(\F) ] = \EE[ G_n(\F) ]$ the \emph{Gaussian average} or the \emph{Gaussian complexity}.
\end{dfn}

\begin{prop}[{Symmetrization \citep[Theorem~7.10]{Steinwart2008}}] \label{prop:sym}
    Let $(\Z,\PP)$ be a probability space, and let $\G \subset L^\infty(\Z)$ be a separable Carathéodory set. Then,
    \begin{align}
        \EE_{D_n}\left[ \sup_{g \in \G} \big| (\PP_n-\PP) g \big|\right]
        \le 2 \rad_n(\G).
    \end{align}
\end{prop}

\begin{prop}[{Chaining, or Dudley's entropy integral \citet[Theorem~2.3.7]{Gine2015}}]
Let $(T,d)$ be a pseudo-metric space, and let $(X_t)_{t \in T}$ be a separable sub-Gaussian process with respect to $d$.
Assume that 
\begin{align}
    \int_0^\infty \sqrt{\log N(T,d,\epsilon)} \dd \epsilon < \infty.
\end{align}
Then, $X$
satisfies
\begin{align}
    \EE  \sup_{t \in T} |X(t)|  &\le \EE|X(t_0)| + 4 \sqrt{2} \int_0^{\diam(T)/2} \sqrt{\log 2N(T,d,\epsilon)} \dd \epsilon\\
    \EE \sup_{\substack{s,t \in T \\ d(s,t) \le \delta}} |X(t) - X(s)|  &\le (16 \sqrt{2} + 2) \int_0^\delta \sqrt{\log 2N(T,d,\epsilon)} \dd \epsilon,
\end{align}
for any $t_0 \in T$ and $\delta > 0$.
\end{prop}

Recall that $( \sqrt{n} R_n(f) )_{f \in \F}$ is a sub-Gaussian process relative to $(\F,\| \cdot \|_n)$. Thus, we have the following.
\begin{prop}[{Dudley's entropy integral bound for Rademacher Complexity \citep[Theorem~7.13]{Steinwart2008} and \citep[Corollary~2.8.8]{Vaart1996}}] \label{prop:dudley}
\begin{align}
    \radhat_n(\F) &\le \frac{4 \sqrt{2}}{\sqrt{n}} \int_0^{\diam(\F)/2} \sqrt{\log 2N(\F,\| \cdot \|_n,\epsilon)} \dd \epsilon,\\
    \EE \sup_{\substack{\| f-g\|_n \le \delta}} |R_n(f-g)|  &\le \frac{16 \sqrt{2} + 2}{\sqrt{n}} \int_0^\delta \sqrt{\log 2N(\F,\| \cdot \|_n,\epsilon)} \dd \epsilon,
\end{align}
for any $f_0 \in \F$ and $\delta > 0$.
\end{prop}

\begin{prop}[{Gaussian concentration inequality \citep[Theorem~2.5.8]{Gine2015}}]
Let $(X_t)_{t \in T}$ be a separable centred Gaussian process,
let $\sigma^2 := \sup_{t \in T} \EE X^2(t)$, and write $G(T) := \sup_{t \in T}|X(t)|$. Then,
\begin{align}
    &\PP\{ G(T) \ge \EE G(T) + u\} \le \exp\left( -\frac{u^2}{2 \sigma^2}\right), \\
    &\PP\{ G(T) \le \EE G(T) - u\} \le \exp\left( -\frac{u^2}{2 \sigma^2}\right).
\end{align}
\end{prop}

Recall that $( \sqrt{n} G_n(f) )_{f \in \F}$ is a separable centred Gaussian process relative to $(\F, \| \cdot \|_n)$.
\begin{prop}[{Gaussian concentration inequality for Gaussian Complexity}] \label{prop:gci}
    \begin{align}
        \PP\left\{ \xi \in \RR^n \,\Bigg|\, \sup_{f \in \F} | G_n(f) | \ge \EE\left[ \sup_{f \in \F} | G_n(f) | \right] + r\right\} \le \exp\left( -\frac{n r^2}{2 \sigma^2 v_n(\F)^2}\right),
    \end{align}
    where $v_n(\F)^2 := \sup_{f \in \F} \| f \|_n^2$.
\end{prop}

\begin{prop}[{Bernstein's inequality \citep[Theorem~3.1.7]{Gine2015}}] \label{prop:bernstein}
Let $X_i \sim P(0,\sigma^2)$ i.i.d. $n$-variables and assume that $\max_i |X_i| \le M$ for some constant $M>0$. Then, for any $t>0$,
\begin{align}
    \PP\left\{ \sum_{i=1}^n X_i \ge t \right\} \le \exp\left( - \frac{t^2/2}{n \sigma^2 + Mt/3}\right).
\end{align}
As a corollary, if $X_i \sim P(\mu,\sigma^2)$ i.i.d. and $\max_i |X_i - \mu| \le M$ for some $M>0$, then
\begin{align}
    \PP\left\{ \frac{1}{n} \sum_{i=1}^n X_i \le \mu + \frac{M u}{3n} + \sqrt{\frac{2 \sigma^2 u}{n}} \right\} \ge 1 - \exp(-u),
\end{align}
for any $u \ge 0$.
\end{prop}

\section{{Proof of \refthm{gen.err}}} \label{sec:proof.gen.err}

\subsection{Assumptions} \label{sec:proof.gen.err.assums}

\begin{itemize}
    \item For the sake of simplicity, we consider the unitary kernel: $K = U$, which leads to $\| S[\mu] \|_2 = \| K[\mu] \|_K$ for any $\mu \in \M_K$
    \item input space: $\X \subset \RR^m$ with data distribution $\PP_X$
    \item parameter space: $\V \subset \RR^d$ bounded by $\sup_{v \in V} |v| \le D_\V$ and $\sup_{v \in \V} \| K_v \|_K \le D_K$
    \item feature map $\varphi : \X \times \V \to \CC$
    uniformly $L_\varphi$-Lipschitz continuous in $\V$, i.e. $\sup_{x \in \X} | \varphi(x;v) | \le L_\varphi |v|$ (required in covering arguments)
    with constant $C_{K \to \infty} := \| \varphi \|_{K \to L^\infty} < \infty$ (implies $\| S [\mu] \|_\infty \le C_{K \to \infty} \| K[\mu] \|_K$)
    \item parameter distribution: $\M^o = \M(\lambda) = \{ \mu \in \M_K \mid \| K[\mu] \|_K \le \lambda \}$ a collection of signed/complex/vector-valued Borel measures that satisfies the kernel moment conditions
    \item[] By the assumptions, $\Ro_\infty := \sup_{\mu \in \M^o} \| S[\mu] \|_\infty = C_{K \to \infty} \lambda < \infty$
    \item hypothesis class $\Fo := \{ S[\mu] \mid \mu \in \M^o \}$
    \item[] By the assumption that $\| S[\mu] \|_\infty < \infty$, $\Fo \subset L^2(\PP)$
    \item search space $\Fp := \{ S[\mu_p] \mid \mid \mu_p \in \M(p) \}$ and $\M(p) := \{ \sum_{j=1}^p w_j \delta_{v_j} \in \M^o \mid \sum_{j=1}^p |w_j| \le D_{TV}, v_j \in \V\}$
    \item[] By the assumptions, $R_\infty := \sup_{\mu \in \M(p)} \| S[\mu] \|_\infty = C_{K \to \infty} \min\{ \lambda, D_{TV} D_K \}$
    \item regression model: $Y = S[\mu^o](X) + \xi$ with $X \sim \PP_X$, $\mu^o \in \M^o$ and $\xi \sim N(0,\sigma^2)$
    \item loss function: $\ell(f;x,y) := \frac{1}{2}|f(x)-y|^2$
    \item expected risk function: $L[f]
    := \frac{1}{2}\PP[ |f(X) - Y|^2 ]
    = \frac{1}{2} \| f - f^o \|_2^2 +  \frac{1}{2}\sigma^2$ 
    \item empirical risk function: $\Lhat[f]
    := \frac{1}{2}\PP_n[ |f(X) - Y|^2 ]
    = \frac{1}{2} \| f - f^o \|_n^2 + \iprod{\xi,f^o-f}_{n'} + \frac{1}{2}\ev_n$
    with $\| f \|_n^2 := \frac{1}{n} \sum_{i=1}^n |f(x_i)|^2, \iprod{\xi,f}_{n'} := \frac{1}{n} \sum_{i=1}^n \xi_i f(x_i)$ and $\ev_n := \frac{1}{n}\sum_{i=1}^n \xi_i^2$
    \item $f^* := \argmin_{f \in \Fp} L[f]$ (expected risk minimizer)
    \item $\fhat := \out(D_n, p)$ always lies in the search space: $\fhat \in \Fp$ 
    \item[] We assume the minimality condition $\Lhat[\fhat] \le \Lhat[f^*]$, which is often assumed in the standard arguments of empirical risk minimization (ERM).
\end{itemize}

\subsection{Proof}

The proof is based on \citet{Suzuki2017}.
Our final estimate is the excess risk $\| \fhat - f^o \|_2^2$. 
By the triangle inequality, it is estimated as
\begin{align}
    \frac{1}{2} \| \fhat - f^o \|_2^2 \le \| \fhat - f^* \|_2^2 + \| f^* - f^o \|_2^2.
\end{align}
Here, $\err_{est} := \frac{1}{2} \| \fhat - f^* \|_2^2$ is the estimation error;
and $\err_{app} := \frac{1}{2} \| f^* - f^o \|_2^2$ is the approximation error, which is discussed in \refthm{approx}.
The estimation error will be estimated by its empirical version $\| \fhat - f^* \|_n^2$ via Talagrand's inequality (\reflem{estimate.n}).
On the other hand,
the approximation error will be used to estimate its empirical version $\| f^* - f^o \|_n^2$ via Bernstein's inequality (\reflem{approx.n}).
The empirical estimation error will be further bounded by the triangle inequality
\begin{align}
    \frac{1}{2} \| \fhat - f^* \|_n^2 \le \| \fhat - f^o \|_n^2 + \| f^* - f^o \|_n^2.
\end{align}
Finally, $\| \fhat - f^o \|_n^2$ will be estimated by the empirical approximation error and the Gaussian complexity.
\begin{align}
\| \fhat - f^o \|_n^2 \le \| f^* - f^o \|_n^2 + 2 G_n(f^* - \fhat). \label{eq:total.n}
\end{align}
This is subject to the minimality assumption: $\Lhat[\fhat] \le \Lhat[f^*]$.
\begin{proof}
Recall that $\Lhat[f] = \frac{1}{2}\| f-f^o\|_n^2 + G_n(f-f^o) + \frac{1}{2} \sum_{i=1}^n \xi_i^2$ for any $f \in \Fo$ conditioned on both $x_i$ and $\xi_i$. Therefore, $\Lhat[\fhat] \le \Lhat[f^*]$ implies that $\frac{1}{2}\| \fhat-f^o\|_n^2 + G_n(\fhat-f^o) \le \frac{1}{2}\| f^*-f^o\|_n^2 + G_n(f^*-f^o)$, which yields the assertion \refeq{total.n}.
\end{proof}
To sum up, we will upper-bound the excess risk by the approximation error and the empirical estimation error.

\subsubsection{Gaussian complexity}
First, we will estimate the Gaussian complexity (\refeq{gss.comp}).
Define 
\begin{align}
    \G(\delta) &:= \{ f - f^* \mid \| f - f^* \|_n \le \delta, f \in \Fp \}.
\end{align}
According to the Gaussian concentration inequality (\refprop{gci}) for $v_n^2 = \sup_{g \in \G(\delta)} \| g \|_n^2 = \delta^2$,
\begin{align}
    \PP\left\{ \sup_{g \in \G(\delta)} | G_n(g) | \ge \EE\left[ \sup_{g \in \G(\delta)} |G_n[g]| \right] + r\right\} \le \exp\left( -\frac{n r^2}{2 \sigma^2 \delta^2}\right),
\end{align}
for any $r > 0$, where the expectation is with respect to $\xi \sim N(0,I_n)$ in $G_n$.
By applying this inequality for $\delta_j = 2^{j-1}\sigma/\sqrt{n} \, (j = 1, \ldots, \lceil \log_2 R_\infty \sqrt{n}/\sigma \rceil)$ and using the uniform bound, we can show that
\begin{align}
    | G_n( f - f^* ) | 
    &\le \EE\left[ \Bigg| \sup_{g \in \G(2 \delta)} G_n(g)\Bigg| \right]
    + 2 \delta \sigma \sqrt{\frac{1}{n}\left( r + \log \lceil \log_2( R_\infty \sqrt{n}/\sigma )\rceil \right)} \\
    &\le \EE\left[ \Bigg| \sup_{g \in \G(2 \delta)} G_n(g)\Bigg| \right]
    + \frac{\mu \delta^2}{2} + \frac{2 \sigma^2}{n \mu} \left( r + \log \lceil \log_2( R_\infty \sqrt{n}/\sigma )\rceil \right).
\end{align}
for any $\mu >0$, and $f \in \Fp$ such that $\| f-f^* \|_n \le \delta$ for any $\delta > \sigma / \sqrt{n}$,
with probability at least $1 - e^{-r}$ for any $r > 0$. 

\begin{lem}[Gaussian complexity]
Let
\begin{align}
    \Phi(n,\delta) := \frac{pd}{n}\log_+\left( 1 + \frac{4 L_\varphi \max\{ D_{TV}, D_\V \} }{\delta} \right).
\end{align}
Then, there exists a universal constant $\Cgc>0$ such that for any $\mu>0$,
\begin{align}
    \EE\left[ \Bigg| \sup_{g \in \G(2 \delta)} G_n(g)\Bigg| \right]
    & \le \frac{\mu \delta^2}{2} + \frac{\sigma^2 \Cgc^2}{2 \mu} \Phi(n,\delta).
\end{align}
\end{lem}
\begin{proof}
Recall that $f \mapsto G_n(f)$ is sub-Gaussian relative to the pseudo-metric $\| \cdot \|_n$.
Thus, by the chaining argument \citep[Ch.8]{Vershynin2018}, we have
\begin{align}
\EE\left[ \Bigg| \sup_{g \in \G(2 \delta)} G_n(g) \Bigg|\right]
    &\le 4 \sqrt{2} \frac{\sigma}{\sqrt{n}} \int_0^{2 \delta} \sqrt{ \log 2N( \G(2 \delta), \| \wdot \|_n, \epsilon ) } \dd \epsilon.
\end{align}
Similarly as \citet[Eq.21]{Suzuki2017}, we have the covering number bound
\begin{align}
\log N(\G(2 \delta ), \| \cdot \|_n, \epsilon)
\le \log N(\Fp, \| \cdot \|_\infty, \epsilon)^2
\le 2 \Phi(n,\epsilon).
 \label{eq:covering}
\end{align}
Hence, there exists a universal constant $\Cgc>0$ such that
\begin{align}
\int_0^{2 \delta} \sqrt{ \log 2N( \G(2 \delta), \| \wdot \|_n, \epsilon ) } \dd \epsilon
& \le \Cgc \delta \sqrt{ \Phi(n,\delta) }.
\end{align}
This asserts the claim.
\end{proof}
By substituting $\delta \gets \left( \| f - f^* \|_n \vee \sigma/\sqrt{n} \right)$ and $r \gets \sigma r / \sqrt{n}$,
\begin{align}
    |G_n(f-f^*)|
    & \le \mu \delta^2
    + \frac{\sigma^2 \Cgc^2 }{2 \mu} \Phi(n,\delta)
    + \frac{2 \sigma^2}{n \mu} \left( r + \log \lceil \log_2( R_\infty \sqrt{n}/\sigma )\rceil \right) \\
    & \le \mu  \| f - f^* \|_n^2
    + \Psi(n,r,\mu), \label{eq:gss.comp}
\end{align}
where
\begin{align}
    \Psi(n,r,\mu)
    :=
    \frac{\sigma^2 \Cgc^2}{2 \mu} \Phi(n,\delta)
    + \frac{2 \sigma^2}{n \mu} \left( r + 2 \mu^2 + \log \lceil \log_2( R_\infty \sqrt{n}/\sigma )\rceil \right).
\end{align}

Combining \refeq{total.n} with $\| \fhat - f^* \|_n^2 \le 2 \| \fhat - f^o \|_n^2 + 2 \| f^o - f^* \|_n^2$ and the Gaussian complexity estimate \refeq{gss.comp} for $\mu = 1/8$, we have
\begin{align}
    \| \fhat - f^* \|_n^2
    &\le 8 \| f^* - f^o \|_n^2 + 4 \Psi(n,r,1/8) + \frac{\sigma^2 pd}{n}.
\end{align}

\subsubsection{Estimation Error}
Next, we will bound a population risk $\| \fhat - f^* \|_2^2$.
\begin{lem} \label{lem:estimate.n}
Let $\Phi_n := \Phi(n, R_\infty/\sqrt{n})$. Then, there exists a constant $\Crc{}>0$ such that for any $f \in \Fp$ such that $\| f - f^* \| \le \delta$,
\begin{align}
    | \| f - f^* \|_n^2 - \| f - f^* \|_2^2 |
    &\le \frac{\delta^2}{2} + \Crc{} R_\infty^2 \left( \Phi_n + \frac{r}{n} \right),
\end{align}
with probability at least $1 - \lceil \log_2 \sqrt{n} \rceil e^{-r}$.
\end{lem}

In the following, we introduce
\begin{align}
\H(\delta) := \{ g^2 - \PP[g^2] \mid g \in \G'(\delta) \},
\end{align}
with an auxiliary space
\begin{align}
    \G'(\delta) &:= \{ f - f^* \mid \| f - f^* \|_2 \le \delta, f \in \Fp \};
\end{align}
so that $\PP_n g^2 = \| f-f^* \|_n^2$ and $\PP g^2 = \| f-f^* \|_2^2$.

\begin{proof}
By construction, (i) $\PP h = 0$ for any $h \in \H(\delta)$.
Since $\| f \|_\infty \le R_\infty$ for any $f \in \Fp$, we have $\| g \|_\infty \le 2 R_\infty $ for any $g \in \G'(\delta)$, and thus (ii) $\sup_{h \in \H(\delta)} \PP h^2 \le 4 R_\infty^2 \delta^2 =: R^2$. Furthermore, we have $\sup_{h \in \H(\delta)}\| h \|_\infty \le \sup_{g \in \G(\delta)} \| g^2 \|_\infty \le 8 R_\infty^2 =: B$. Therefore, Talagrand's concentration inequality (\refprop{talagrand}) with $\gamma = 1$ yields
\begin{align}
    \sup_{g \in \G'(\delta)} |\PP_n g^2 - \PP g^2| &\le 2 \EE_{D_n} \left[ \sup_{g \in \G'(\delta)} |\PP_n g^2 - \PP g^2|\right] + \sqrt{\frac{2 t R^2}{n}} + \frac{5tB}{3n}, \label{eq:talagrand.h}
\end{align}
for any $n \in \NN$, with probability at least $1-e^{-t}$ for any $t>0$.

By the symmetrization argument (\refprop{sym}), and the standard Rademacher complexity calculus,
\begin{align}
    \EE_{D_n} \left[ \sup_{g \in \G'(\delta)} |\PP_n g^2 - \PP g^2|\right]
    &\le 2 \rad_n( \{ g^2 \mid g \in \G'(\delta) \} ) \\
    &\le 4 (2 R_\infty) \rad_n( \G'(\delta) ).
\end{align}
Following \citet[Lemma~2.3]{Mendelson2002} with the covering number bound \refeq{covering}, there exists a constant $\Crc{.0}>0$ such that
\begin{align}
    \rad_n( \G'(\delta) )
    &\le \Crc{.0} [ \delta \sqrt{\Phi(n,\delta)} \vee R_\infty \Phi(n,\delta)].
\end{align}
Letting $\Phi_n := \Phi(n,R_\infty/\sqrt{n})$,
and applying the inequality \refeq{talagrand.h} for $\delta_j = 2^{j-1}R_\infty/\sqrt{n} (j=1,\ldots,\lceil \log_2 \sqrt{n} \rceil)$, we can show that there exists an event with probability $1 - \lceil \log_2 \sqrt{n} \rceil e^{-r}$ such that uniformly for all $g \in \G'(\delta)$,
\begin{align}
    |\PP_n g^2 - \PP g^2| 
    &\le 
    8 \Crc{.0} [ \delta R_\infty \sqrt{\Phi_n} \vee R_\infty^2 \Phi_n] + 2 \delta \sqrt{\frac{2 t R_\infty^2}{n}} + \frac{40 t R_\infty^2}{3n} \\
    &\le \frac{1}{2}( \delta^2 \vee R^2_\infty / n ) + (\Crc{.0}+\Crc{.0}^2) R_\infty^2 \Phi_n + \Crc{.1} R_\infty^2 r / n \\
    &\le \frac{\delta^2}{2} + \Crc{.2} R_\infty^2 \left( \Phi_n + \frac{r}{n}\right),
\end{align}
for some constants $\Crc{.1}$ and $\Crc{.2}$, 
where $\delta$ is any positive number such that $\delta^2 > \PP g^2 \vee R_\infty^2/n$.
\end{proof}

Thus far, by letting $f \gets \fhat$ and $\Crc{} \gets \Crc{.2}$, 
\begin{align}
    \frac{1}{2}\| \fhat - f^* \|_2^2 & \le \Crc{} R_\infty^2 \left( \Phi_n + \frac{r}{n} \right) + 8 \| f^* - f^o \|_n^2 + 4 \Psi(n,r,1/8) + \frac{\sigma^2 + R_\infty^2}{n}pd.
\end{align}

\subsubsection{Approximation Error}
Similarly as \reflem{estimate.n}, we estimate the empirical approximation error via Bernstein's inequality (\refprop{bernstein}).
\begin{lem} \label{lem:approx.n}
With probability at least $1-e^{-u}$ for any $u > 0$, we have
\begin{align}
    \| f^* - f^o \|_n^2
    &\le \frac{3}{2} \| f^* - f^o \|_2^2 + \frac{2 \| f^* - f^o \|_\infty^2 u}{n}.
\end{align}
\end{lem}
\begin{proof}
For the sake of visibility, write $g := f^* - f^o$ and let $Y := g^2(X)$. Then, we have $Y \sim P( \| g \|_2^2, \sigma^2 )$ with $\sigma^2 = \EE[ Y^2 ] - \EE[Y]^2 \le \EE[ g^4(X) ] \le \| g \|_\infty^2 \| g \|_2^2$ and $|Y - \EE[Y] | \le M$ with $M \le \| g^2(x) - \| g \|_2^2 \|_\infty \le \| g \|_\infty^2 + \| g \|_2^2 \le 2 \| g \|_\infty^2$.
Therefore, we have
\begin{align}
    \| g \|_n^2
    &\le \| g \|_2^2 + \frac{\| g \|_\infty^2 u}{3n} + \sqrt{\frac{2 \| g \|_\infty^2 \| g \|_2^2 u}{n}},
\end{align}
with probability at least $1-e^{-u}$ for any $u \ge 0$.
By the AM-GM inequality, we have $\sqrt{2 \| g \|_\infty^2 \| g \|_2^2 u /n} \le \| g \|_\infty^2 u/n + \| g \|_2^2/2$, which yields the claim.
\end{proof}

\subsubsection{Excess Risk}
Combining \reflem{approx.n} with a triangle inequality $\| \fhat - f^o \|_2^2 \le 2 \| \fhat - f^* \|_2^2 + 2 \| f^* - f^o \|_2^2$, we have the final estimate
\begin{align}
    \frac{1}{2}\| \fhat - f^o \|_2^2
    &\le 2 \Crc{} R_\infty^2 \left( \Phi_n + \frac{r}{n} \right) + 25 \| f^* - f^o \|_2^2 + 48\frac{\Ro_\infty^2 u}{n} + 8 \Psi(n,r,1/8) + \frac{2(\sigma^2 + R_\infty^2)}{n}pd \\
    &= 2 \Crc{} \frac{R_\infty^2 pd}{n} \log_+\left( 1 + \frac{4 \Clog \sqrt{n}}{R_\infty}\right) + 36 \Cgc^2 \frac{\sigma^2 pd}{n} \log_+\left( 1 + \frac{4 \Clog \sqrt{n}}{\sigma}\right) + \frac{2(\sigma^2 + R_\infty^2)pd}{n} \tag{est. error} \\
    &\quad 
    + 25 \| f^* - f^o \|_2^2 \tag{approx. error} \\
    &\quad + 2 \Crc{} \frac{R_\infty^2 r}{n}
    + \frac{\sigma^2}{n}\left( 16r + \frac{1}{2} + \log \left\lceil \log_2 \frac{R_\infty \sqrt{n}}{\sigma} \right\rceil \right)
    + 48 \frac{\Ro_\infty^2 u}{n}, \tag{residual} \\
    &\lesssim \frac{(R_\infty^2 + \sigma^2)pd}{n}\log_+\left( 1 + \frac{4 \Clog \sqrt{n}}{R_\infty \wedge \sigma}\right) + \| f^* - f^o \|_2^2 + O_p\left( \frac{R_\infty^2 + \sigma^2}{n} \right),
\end{align}
where $\Clog = L_\varphi \max\{ D_{TV}, D_\V\}$, with probability at least $1 - e^{-u} - e^{-r}$ with any $u \ge 0$ and $r \ge 0$.
}

\section{Details on Experimental Results}

\subsection{Setup}

\paragraph{Model.} We specify the parameter space $\V = \RR^m \times \RR$ and rewrite $v (\in \V)$ as $(a,b)$.
We consider a shallow neural network $S[\mu](x) = \int_{\RR^m \times \RR} \sigma( a \cdot x - b ) \dd \mu(a,b)$, and the activation function $\sigma$ is the first derivative of Gaussian kernel.

\paragraph{Methods.} We compare two methods: (SIR) Sampling importance resampling, or an i.i.d. sampling of parameters $v_j =(a_j, b_j)$ from the probability distribution $Q(a,b)$ that is proportional to the ridgelet spectrum $|R[f](a,b)|$, which is a pseudo-inverse operator ($S^\dag$) of $S$ defined by $R[f](a,b) := \int_{\RR^m} f(x) \sigma( a \cdot x - b ) \dd x$;
and (UKQ) unitary kernel quadrature, or greedy minimization of the MMD with the unitary kernel.
SIR corresponds to an ordinary Monte Carlo method developed in \citep{Sonoda2014}, and UKQ corresponds to the proposed method. In both methods, we determined weights $w_j$ by using linear regression. See \refsec{sir} for more details on SIR.

\paragraph{Datasets.} We present the results with sinusoidal curve with Gaussian noise: $y_i = \sin 2 \pi x_i + \xi_i$ with $x_i \sim \mathrm{Uni}[-1,1], \xi_i \sim N(0,0.1^2)$ and $n = 100$.
We used another $n = 100$ sample for evaluation.

\paragraph{Evaluation.} We employed the empirical maximum error (ME) $\max_i |y_i - S[\mu_p](x_i)|$, and the root mean squared error (RMSE) $\sqrt{ \frac{1}{n} \sum_{i=1}^n | y_i - S[\mu_p](x_i) |^2 }$. 

\subsection{Details on Sampling Importane Resampling (SIR)} \label{sec:sir}
The SIR method assumes that we know the pseudo-inverse operator $S^\dag$,
In addition, let $Q$ be the proposal distribution, or a probability distribution on $\V$ from which we can easily draw samples.

First, we draw $p' ( \gg p)$ samples $\{ v_j' \}_{j=1}^{p'}$ from a proposal probability density $Q$ on $\V$.
Second, we compute the importance weights $w_j' := S^\dag(v_j') / Q(v_j')$.
Then, we resample $p$ elements $A_p := \{ v_j \}_{j=1}^p$ from  $\{ v_j' \}_{j=1}^{p'}$ with occurence probability $P(v_j') := w_j' / \sum_{j=1}^{p'} w_j'$.
Here, the resampling is conducted \emph{without} replacement to avoid multiple $v_j$'s to indicate the same feature.
The resampled set $A_p$ asymptotically distributed according to the spectrum density $|S^\dag(v)|/Z$.
Here, $Z := \int_{\V} |S^\dag|(v) \dd v = \| \mu \|_{TV}$ is the normalization constant, which we do not need to compute in SIR.
Finally, we determine the weight $w_j$ for every $v_j \in A_p$ by either (1) letting $w_j := S^\dag(v_j)/p$ or (2) minimizing $\| S[\mu] - S[\mu_p]\|_{L^2(\PP)}^2$ with $\mu_p = \sum_{j=1}^p w_j \delta_{v_j}$. 
By the uniform law of large numbers, $S[\mu_p]$ converges to $S[\mu]$ at $O(1/\sqrt{p})$ in $L^\infty(\X)$ (\refthm{ulln.sir}).

\end{document}